\documentclass[twoside]{article}

\usepackage[accepted]{aistats2024}

\usepackage{hyperref}       
\usepackage{url}            
\usepackage{amsmath}
\usepackage{amsfonts}       

\usepackage[round]{natbib}

\usepackage{xspace}
\usepackage{algorithm}
\usepackage{algpseudocode}
\usepackage{std_definitions}

\newcommand{\unif}[0]{\textup{unif}\xspace}

\newcommand{\epstv}[0]{\epsilon_\textup{TV}\xspace}
\newcommand{\epsescape}[0]{\epsilon_\textup{escape}\xspace}
\newcommand{\jmax}[0]{J_\textup{max}\xspace}

\begin{document}

\twocolumn[

\aistatstitle{Low-rank MDPs with Continuous Action Spaces}

\aistatsauthor{ Andrew Bennett$^\dagger$ \And Nathan Kallus$^\dagger$ \And  Miruna Oprescu$^\dagger$}

\aistatsaddress{ Morgan Stanley \And  Cornell Tech \And Cornell Tech }
]

\begin{abstract}
 Low-Rank Markov Decision Processes (MDPs) have recently emerged as a promising framework within the domain of reinforcement learning (RL), as they allow for provably approximately correct (PAC) learning guarantees while also incorporating ML algorithms for representation learning. However, current methods for low-rank MDPs are limited in that they only consider finite action spaces, and give vacuous bounds as $|\mathcal{A}| \to \infty$, which greatly limits their applicability.
  In this work, we study the problem of extending such methods to settings with continuous actions, and explore multiple concrete approaches for performing this extension.
  As a case study, we consider the seminal FLAMBE algorithm \citep{agarwal2020flambe}, which is a reward-agnostic method for PAC RL with low-rank MDPs. We show that, without any modifications to the algorithm, we obtain a similar PAC bound when actions are allowed to be continuous. Specifically, when the model for transition functions satisfies a H\"older smoothness condition w.r.t. actions, and either the policy class has a uniformly bounded minimum density or the reward function is also H\"older smooth, we obtain a polynomial PAC bound that depends on the order of smoothness. 
\end{abstract}

\section{INTRODUCTION}

Reinforcement learning (RL) is an important framework for the study of automated decision making, where decisions are sequentially made over time with context that evolves based on previous decisions. While RL gained popularity through large scale game-playing,   
this problem has many diverse practical applications, such as memory management for computing systems \citep{chinchali2018cellular}, online resource allocation in computer clusters \citep{lykouris2018competitive, tessler2021reinforcement}, personalized recommendations \citep{zheng2018drn}, bidding/advertising \citep{cai2017real}, autonomous driving \citep{sallab2017deep}, robotics \citep{kober2013reinforcement}, and automated trading \citep{li2019deep}.

It is widely recognized that achieving provably correct RL for general Markov Decision Processes (MDPs) with arbitrary state spaces incurs a sample complexity of either $\Omega(|\Scal|)$ or $\Omega(|\Acal|^H)$, where $|\Scal|$, $|\Acal|$, and $H$ are the sizes of the state space, action space, and horizon respectively \citep{krishnamurthy2016pac}. This is problematic, because many of the most promising RL applications involve both long horizons and complex observations such as images or text, so both $|\Scal|$ and $|\Acal|^H$ may be intractably large. Fortunately, there has been extensive recent work on discovering more restricted MDP classes that are useful for modeling realistic problems with rich state spaces while still allowing for PAC learning results with polynomial sample complexity \citep{jiang2017contextual, du2019provably, sun2019model, jin2020provably, agarwal2020flambe, misra2020kinematic, du2021bilinear, chen2022statistical}.

A framework of particular interest is low-rank MDPs \citep{jiang2017contextual}, since they are extremely general and subsume many other such frameworks (\emph{e.g.}, linear MDPs, \citealp{jin2020provably}, or block MDPs, \citealp{du2019provably}). Unlike some of the more general frameworks, such as bounded Bellman-rank MDPs \citep{jiang2017contextual}, low-rank MDPs allow for the development of efficient and practical algorithms. This framework utilizes a low-rank structure for the MDP transitions based on state and action features, which can be estimated using conventional machine learning techniques, rendering them particularly suited for addressing challenges in environments with rich and complex action spaces. Despite the rich body of recent research on PAC algorithms for low-rank MDPs with polynomial sample complexity, the complexity invariably scales polynomially with $|\Acal|$. This is problematic, because many of the real world applications that motivate this line of work feature large or infinite action spaces; for example, actions in autonomous driving have continuous parts (\emph{e.g.}, steering wheel angle), and actions in algorithmic trading have (roughly) continuous parts (\emph{e.g.}, quantity of security to buy/sell). Furthermore, the low-rank RL framework allows for flexible modeling of generic action spaces via the state/action embedding function, so it seems natural that it should be able to tractably handle infinitely large action spaces.

In this paper, we study the problem of extending low-rank MDP PAC results to settings with continuous action spaces, where $|\Acal| = \infty$. We first discuss why the dependence on $|\Acal|$ occurs with these algorithms, and how this dependence may be alleviated. In particular, we provide a novel result motivated by Sobolev interpolation theory that can be used to improve existing analyses when certain problem parameters are sufficiently smooth. Then, as a specific case study, we consider the FLAMBE algorithm of \citet{agarwal2020flambe}, which was the first tractable (computationally oracle efficient) PAC RL algorithm for low-rank MDPs, and allows for reward-agnostic learning. We discuss how our methods can be applied to the analysis of FLAMBE, in order to provide PAC bounds for continuous action spaces. These bounds apply when the transition operators allowed by the model are smooth w.r.t. actions, and either the policy class has a uniform bounded maximum density or the allowed reward functions are H\"older-continuous in actions. In addition, we provide a similar analysis for the more recent RAFFLE algorithm of \citet{cheng2023improved}, which was shown to have near-optimal sample complexity, in the appendix.
Aside from the specific continuous action PAC bounds established for FLAMBE and RAFFLE, we hope this work will serve as a valuable template for generalizing low-rank MDP PAC algorithms to continuous action spaces more broadly.

\subsection{Related Work}
\label{sec:related-work}

\textbf{Low-rank MDPs} Low-rank MDPs are a class of models at the interplay of representation learning and exploration in reinforcement learning. In recent years, these models have garnered a lot of interest due to efforts to provide algorithms that can leverage machine learning while also having provable guarantees on \emph{e.g.} correctness and sample complexity  \citep{jin2020provably, yang2019sample, yang2020reinforcement, agarwal2020flambe, modi2021model, uehara2021representation, chen2022statistical, cheng2023improved}.
The seminal oracle-efficient algorithm in this space is FLAMBE \citep{agarwal2020flambe}, which performs reward-free exploration to learn a low-rank model of the MDP. Subsequent works by have developed other algorithms improving on this in various ways, such as establishing tighter bounds for exploring with a particular reward function \citep{uehara2021representation}, performing both model-free and reward-free exploration \citep{modi2021model}, extending to MDPs with low Bellman Eluder dimension \citep{chen2022statistical}, or using exploration-driven proxy rewards \citep{cheng2023improved}. In addition to oracle-efficient algorithms for low-rank MDPs, there is past work on (computationally inefficient) algorithms for more general bounded Bellman/Witness Rank classes \citep{jiang2017contextual,sun2019model,du2021bilinear}, or for the more restricted block MDP class \citep{du2019provably,misra2020kinematic}. However, these works generally do not allow for continuous action spaces.

One exception to the above is \citet{agarwal2022non}, which considers an even more general setting (linear embeddability), and allows for arbitrary action spaces. However, this approach has two important drawbacks. First, as is common in the literature that considers similarly general settings such as bounded Bellman rank, the algorithms presented are at face value computationally intractable; in the case of \citet{agarwal2022non}, their proposed methodology requires posterior sampling over all possible $Q$ functions. In contrast, our paper follows the line of work on the more restricted low-rank MDP setting, for which tractable algorithms (given an oracle for performing MLE) can be designed. Second, given the lower bound in \citet{cheng2023improved} for low-rank MDPs that depends on $|\Acal|$, it is clear that the dependence on action space size \emph{cannot} be avoided in general without additional assumptions, such as smoothness. Thus, the assumptions in \citet{agarwal2022non} implicitly constrain the low-rank MDP setting in some other way.

\textbf{Continuous action spaces in RL}
Some efforts have been made to adapt PAC RL algorithms designed for discrete action spaces to continuous ones via the discretization of the action space, as seen in the work of \citet{song2019efficient, domingues2020regret, sinclair2021adaptive}. However, discretization has several drawbacks, such as computational inefficiency and exponential performance degradation in high dimensions, due to its limited capacity to exploit function smoothness beyond Lipschitz. In contrast, our approach can leverage any degree of smoothness to counteract the curse of dimension and, as we will see, our bounds approach those for finite action spaces when the smoothness of the transition function grows. This cannot be achieved by discretization. To the best of our knowledge, there has been no other effort in directly adapting such RL algorithms to continuous action spaces without discretizing the action space.

\textbf{Leveraging Smoothness in Machine Learning} Smooth functions (continuous functions that admit higher order derivatives) have become popular in machine learning,
since improved learning rates can often be achieved for sufficiently smooth classes. For example, it is well known that learning $d$-dimensional, $\alpha$-times differentiable functions have a minimax-optimal convergence rate of $O(n^{-2\alpha/(2\alpha+d)})$ \citep{stone1982optimal}.
Many algorithms can leverage such smoothness, including kernel methods, Gaussian processes, and deep neural networks \citep{gine2021mathematical}. While smooth functions have also found some practical application in reinforcement learning, particularly when considering smooth rewards in multi-armed bandit ($H=1$) problems \citep{liu2021smooth, jia2023smooth,hu2022smooth}, its application to reinforcement learning more broadly remains a promising area for further exploration.

\subsection{Overview of Paper}

First, in \pref{sec:discrete-setting} we formally introduce the low-rank MDP setting as in \citet{jiang2017contextual}, and provide a brief overview of existing PAC results for this setting. Then, in \pref{sec:low-rank-limitations} we discuss why the $|\Acal|$ dependence tends to occur for these algorithms and corresponding PAC bounds, and in \pref{sec:continuous-extension} we provide some theoretical results and discussion that are helpful for extending these results to continuous $\Acal$. Expanding upon this, in \pref{sec:flambe} we consider as a case study the FLAMBE algorithm of \citet{agarwal2020flambe}, and we apply our results from \pref{sec:continuous-extension} to obtain PAC bounds for FLAMBE that hold when $\Acal$ is continuous, \emph{without any modifications to the algorithm}. We also provide a similar analysis for the more recent RAFFLE algorithm of \citet{cheng2023improved} in \pref{app:raffle}. Finally, we draw some conclusions in \pref{sec:conclusion}. We relegate all proofs to \pref{app:proofs}.

\subsection{Notation}

We let $\NN$ denote the space of non-negative integers $\{0,1,2,\ldots\}$. We let $\NN^+$ and $\RR^+$ denote the space of positive integers and real numbers, respectively (\emph{i.e.}, $\NN \setminus \NN^+ = \{0\}$). For any $k\in\NN^+$, we let $[k] = \{0,1,\ldots,k-1\}$. For any $s\in\RR^+$, we let $\lfloor s \rfloor=\{i\in\mathbb N:i<s\}$ be the largest integer strictly smaller than $s$ (note $\lfloor 5\rfloor=4$ unlike the usual floor function) and $\{s\}=s-\lfloor s\rfloor$ be the remainder; that is, we have the unique decomposition $s = \lfloor s \rfloor + \{s\}$, where $\lfloor s \rfloor \in \NN$ and $\{s\} \in (0,1]$.
For any $f: \RR^d \rightarrow \RR$ and  $\beta\in \NN^{d}$, we define $D^\beta f = \frac{\partial^{|\beta|}}{\partial x_1^{\beta_1}\partial x_2^{\beta_2}\dots\partial x_{d}^{\beta_{d}}} f(x)$ whenever this derivative exists, where $|\beta|:=\sum_{i=1}^{d}\beta_i$. Whenever all derivatives of order $j$ exist, we let $\nabla^j f$ denote the vector-valued function of all $D^\beta f$ such that $|\beta|=j$, and we let $\nabla^0 f=f$. 
Finally, for any $p\in[1, \infty]$, positive integer $d$, and Borel measurable $\Dcal \subseteq \RR^d$, we let $\|\cdot\|_{L_p(\Dcal)}$ denote the standard $L_p$ norm w.r.t. the Borel measure on $\Dcal$. We abbreviate $\|\cdot\|_{L_\infty(\Dcal)}$ (\emph{i.e.}, the essential supremum of the absolute value) by $\|\cdot\|_\infty$.

\section{LOW-RANK MDP SETTING}
\label{sec:discrete-setting}

We consider a finite-horizon episodic MDP defined by a tuple $\Mcal = (\Scal, \Acal, T, H, \rho)$, where $\Scal$ is the state space, $\Acal$ is the action space, $T = (T_0,\ldots,T_{H-1})$, where $T_h : \Scal \times \Acal \mapsto \Delta(\Scal)$ is the transition function at time $h$, $H$ is the finite time horizon, and $\rho \in \Delta(\Scal)$ is the initial state distribution. We assume that $\Scal$ may be some generic measurable set (with measure $\mu$), such that $T(\cdot \mid s,a)$ is continuous w.r.t. $\mu$ for every $(s,a)$.\footnote{That is, we can interpret the function $T(s' \mid s,a) \in \RR^+$ as the conditional density of $s'$ given $(s,a)$, with respect to base measure $\mu$, which may be, \emph{e.g.}, a Borel measure (if $\Scal$ is continuous) or a counting measure (if $\Scal$ is discrete).} We note that this MDP definition does \emph{not} include reward functions $R_h : \Scal \times \Acal \mapsto [0,1]$ for each $h \in [H]$, since we will consider reward-agnostic algorithms. 
In addition, past definitions of this setting such as in \citet{jiang2017contextual} often dictate that $\Acal = [K]$ for some finite $K \in \NN^+$, but we do not impose this restriction.

The \emph{low-rank} assumption is that the transition operator $T$ has a low-rank representation, and can be summarized by the following definition leveraged from \citet{agarwal2020flambe}:
\begin{definition}[Low-rank MDP]
\label{def:low-rank}
An operator $T : \Scal \times \Acal \mapsto \Delta(\Scal)$ admits a low-rank decomposition with dimension $d \in \NN$ if there exists two embedding functions $\phi^* : \Scal \times \Acal \mapsto \RR^d$ and $\psi^* : \Scal \mapsto \RR^d$ such that
\begin{equation*}
    T(s' \mid s,a) = \phi^*(s,a)^\top \psi^*(s') \qquad \forall \ s,s' \in \Scal \,,
\end{equation*}
where the embeddings $\phi^*$ and $\psi^*$ satisfy the (without loss of generality) normalization conditions $\|\phi^*(s,a)\|_2 \leq 1 \ \forall s,a$, and $\|\int \psi^*(s) g(s) d \mu(s)\|_2 \leq \sqrt{d} \ \forall g$ such that $g : \Scal \mapsto [0,1]$. An MDP $\Mcal$ is a low-rank MDP with rank $d$ if $T_h$ admits a low-rank decomposition with dimension $d$ for every $h \in [H]$. We let $\phi_h^*$ and $\psi_h^*$ denote the true embeddings for the decomposition of $T_h$.
\end{definition}

For learning within a low-rank MDP structure, we assume that we are given a set of hypothesis embeddings $\Phi \subseteq \Scal \times \Acal \mapsto \RR^d$ and $\Psi \subseteq \Scal \mapsto \RR^d$ such that we have the following \emph{realizability} assumption:
\begin{assum}[Realizability]
\label{assum:realizability}
For every $h \in [H]$, we have $\phi_h^* \in \Phi$ and $\psi_h^* \in \Psi$.
\end{assum}

\section{LIMITATIONS OF EXISTING LOW-RANK MDP RESULTS}
\label{sec:low-rank-limitations}

As highlighted in \pref{sec:related-work}, a range of recent algorithms offer PAC guarantees for low-rank MDPs. In general, all of these algorithms have sample complexity that is $\tilde\Ocal(\textup{Poly}(d,H,|\Acal|,\epsilon,\log(1/\delta)))$ for learning $\epsilon$-optimal policies with probability at least $1-\delta$. Additionally, in every instance, these bounds exhibit a super-linear dependence on $|\Acal|$, leading to significantly poor scalability when $|\Acal|$ becomes combinatorially large or infinite.

Fundamentally, the dependence on $|\Acal|$ arises because these algorithms tend to employ some form of uniform exploration, with actions being sampled uniformly at random at least sometimes. In the linear MDP setting, where the $\phi_h^*$ embeddings are known, it is possible to bypass uniform exploration by exploring strategically within the embedding space structure \citep{jin2020provably}, thus eliminating the sample complexity's dependence on $|\Acal|$. However, low-rank MDPs present the added complexity of needing to learn these embeddings. For example, low-rank MDP papers leverage MLE bounds (\emph{e.g.}, see \citealp[Theorem 21]{agarwal2020flambe}) to learn embeddings that are accurate across the distribution of state-action pairs $(s,a)$ that the training data is sampled from: if we ensure that such error is small when $a$ is sampled uniformly at random for a given state $s$, then we can ensure that it is small for \emph{all} policies applied at the same $s$, based on the following importance sampling lemma.
\begin{lemma}
\label{lem:is}
For any given $f : \Scal \times \Acal \to \RR^+$, any distribution $\rho$ over states, and any policy $\pi$, we have
\begin{align*}
    \EE_{s \sim \rho, a \sim \pi(\cdot \mid s)}[f(s,a)] 
    \leq  |\Acal|  \EE_{s \sim \rho, a \sim \unif_{\Acal}}[f(s,a)] \,.
\end{align*}
\end{lemma}
This lemma is trivial to prove by noting that $\pi(a \mid s) / \unif_\Acal(a) \leq 1 / \unif_\Acal(a) = |\Acal|$ for all discrete sets $\Acal$. That is, the importance sampling lemma, which is implicitly used in the analysis of most of these papers, allows us to bound the errors of learned embeddings (or the probabilities of other ``bad'' events) under any action distribution, given that we perform some exploration uniformly at random.

Unfortunately, the implicit use of \pref{lem:is} in the theoretical analysis of these works introduces $|\Acal|$ factors within error bounds, which translates into polynomial dependence on $|\Acal|$ in the sample complexity.

\section{ERROR BOUNDS WITH CONTINUOUS ACTIONS}
\label{sec:continuous-extension}

Given the limitations discussed in \pref{sec:low-rank-limitations}, a natural question to ask is whether we can provide bounds like \pref{lem:is} when $\Acal$ is continuous. In order to make this discussion concrete, we will consider the continuous action space $\mathcal{A}=[0,1]^{m}\subset \RR^{m}$ for some $m \in \NN^+$, which is equipped with the standard Borel measure for the Euclidean space $\RR^{m}$. Although this choice of continuous action space may at a first glance seem restrictive, it is a natural way to parameterize many kinds of continuous actions that have $m$ independent degrees of freedom, and the action space can typically be written this way, \emph{e.g.}, by a simple change of variables. 
Note as well that that this assumed normalization of action components to be on the scale $[0,1]$ is important for interpreting smoothness of functions w.r.t. $a$; for example, if the ``natural'' scale of actions for a given problem were $[0,A_{\textup{max}}]$, for some $A_{\textup{max}} \gg 1$, then we would need to scale up all order-$\alpha$ derivatives w.r.t. actions by a factor of $A_{\textup{max}}^\alpha$ in order to normalize and be scale invariant. Additionally, our results easily extend to more general $\Acal$ with Lipschitz boundaries, as the main property of $[0,1]^m$ used in the proofs of the lemmas and theorem below are that it is Lipschitz. See our proofs in the appendix for more detail.

In the next sections, we explore two strategies to address the challenges outlined. In our case study examining the FLAMBE algorithm of \citet{agarwal2020flambe} in \pref{sec:flambe}, we will use both approaches in different parts of the analysis.

\subsection{Utilizing Smoothness of Error Functions}
\label{sec:method-smoothness}

In many cases where we need to apply \pref{lem:is}, the function $f(s,a)$ corresponds to an ``error'' term associated with state-action pairs. For many problems, we may be able to justify that such errors are smooth with respect to $a$ given, \emph{e.g.}, assumptions on the smoothness of transitions in terms of $a$. Specifically, we will provide a generalization of \pref{lem:is} that can be applied to a fairly general class of smooth functions. 

\begin{definition}[$\alpha$-smooth functions]\label{def:a-smooth}
    For any $\alpha\in(0,\infty)$ and $\Dcal \subseteq \RR^{m}$, $f: \Dcal \rightarrow \RR$ is an $\alpha$-smooth function if the $\lfloor \alpha \rfloor$-order mixed derivatives of $f$ exist, and there exists a constant $L$ such that: (1) $\left\|\nabla^j f \right\|_\infty \leq L$ for all $j\leq {\lfloor \alpha \rfloor }$; and (2) all $\lfloor \alpha \rfloor$-order derivatives are $\{\alpha\}$-H\"older continuous with norm at most $L$, \emph{i.e.},
    \begin{equation*}
        \sup_{a,a' \in \Dcal, |\beta| = \lfloor\alpha\rfloor} \frac{|D^\beta f(a) - D^\beta f(a')|}{\|a - a'\|_2^{\{\alpha\}}} \leq L \,.
    \end{equation*}
\end{definition}
That is, $\alpha$-smooth functions have mixed derivatives up to the order $\lfloor\alpha\rfloor$ that are bounded. Moreover, these mixed derivatives have  have bounded $\{\alpha\}$ fractional derivatives. For example, when $\alpha \in (0,1]$, this definition corresponds to $\alpha$-H\"older continuity \citep{leoni2017first}; and, in particular, when $\alpha=1$, this definition corresponds to Lipschitz continuity. Furthermore, for any $\alpha \in (0,\infty)$ the set of $\alpha$-smooth functions corresponds to the $\alpha$-H\"older space \citep{tsybakov2008}. However, implicit in our definition is a particular norm, which we explicitly define below.
\begin{definition}[$\alpha$-smoothness norm]\label{def:a-norm}
    For any $\alpha$-smooth function $f$, we define $\|f\|_{C^\alpha}$ as the minimum $L \geq 0$ such that the conditions of \pref{def:a-smooth} hold.
\end{definition}
For example, $\|f\|_{C^\alpha}$ corresponds to the H\"older-continuity norm for all $\alpha \leq 1$, and $\|f\|_{C^1}$ corresponds to Lipschitz norm. More generally, for any $\alpha \in (0,\infty)$ we have that $\|\cdot\|_{C^\alpha}$ is a norm on the $\alpha$-H\"older space. We also note that these definitions are clearly nested, with $\|f\|_{C^\alpha} \leq \|f\|_{C^\beta}$ for any $\alpha < \beta$.

Given this smoothness definition, we can provide a weaker version of \pref{lem:is} under the condition that $f$ is $\alpha$-smooth in $a$ for each $s$. This bound is based on the following theorem, which bounds the maximum of $\alpha$-smooth functions by some power of their mean.

\begin{theorem}[Uniform bound on $\alpha$-smooth functions]
\label{thm:is-alpha-smooth}
    Let $f : \Scal \times \Acal \to \RR^+$ be $\alpha$-smooth in the action $a$ such that $\|f\|_{C^\alpha} \leq L$ for any $s\in \Scal$ and some $\alpha \in (0,\infty)$. Then, for any distribution $\rho$ over states, and any policy $\pi$, we have
    \begin{align*}
        &\EE_{s \sim \rho, a \sim \pi(\cdot \mid s)}  [f(s,a)] \\
        & \leq c L^{\frac{m}{m+\alpha}} \Big(\EE_{s \sim \rho, a \sim \unif_\Acal}[f(s,a)]\Big)^\frac{\alpha}{m+\alpha} \,,
    \end{align*}
    for some constant $c$ that depends only on $m$ and $\alpha$.
\end{theorem}
The proof of \pref{thm:is-alpha-smooth} is based on a result that allows us to bound the maximum value of any non-negative, $\alpha$-smooth function on a Lipschitz subset of $\RR^d$ by its $L_1$ norm on that domain to the power of $\alpha / (m+\alpha)$; this follows from Sobolev interpolation theory (see \emph{e.g.}, \citealp{fiorenza2021detailed}) and may be of independent interest. Full proof details are provided in the appendix.

In terms of exponents on the mean, \pref{thm:is-alpha-smooth} gives a weaker version of the inequality in \pref{lem:is}, and the two exponents coincide only as $\alpha\to\infty$. Often, this more relaxed bound can substitute \pref{lem:is}, trading $|\Acal|$ dependency for less favorable exponents in downstream sample complexity calculations.

\subsection{Using Smoothed Policies}
\label{sec:method-boundedness}

An alternative approach to extending \pref{lem:is} to continuous $\Acal$ is to restrict attention to policies with uniformly bounded density. Specifically, if we restrict to policies $\pi$ such that $\pi(a \mid s) / \unif_\Acal(a)$ is bounded by some constant $K$ for all $s \in \Scal$, then it is trivial to verify that \pref{lem:is} holds with $|\Acal|$ replaced by $K$. Therefore, in cases where this restriction in policy class is feasible, we may be able to produce identical PAC results under this restriction with $|\Acal|$ replaced by $K$.

Nevertheless, this restriction may be unsuitable in practice, since we might want to be able to learn approximately optimal deterministic policies, which have unbounded density. Fortunately, in many cases it may be possible to extend PAC results based on this restriction to unrestricted PAC results when reward and transition functions are H\"older continuous, based on the following lemma.

\begin{lemma}
\label{lem:policy-smoothing}
    Let $V(\pi;R) = \EE_\pi[\sum_{h \in H} R(s_h,a_h)]$ denote the value of policy $\pi$ for reward function $R$, where $\EE_\pi$ denotes expectation over data sampled following policy $\pi$ for all $H$ time steps. In addition, assume that there exists some $\alpha \in (0,1]$ and $L < \infty$ such that: (1) $R(s,a)$ is uniformly $\alpha$-H\"older continuous in $a$ with norm at most $L$ all $s \in \Scal$; and (2) $\phi_h^*(s,a)^\top \psi_h^*(\cdot)$ is uniformly $\alpha$-H\"older continuous in $a$ under total variation distance with norm at most $L$ for all $s \in \Scal$ and $h \in [H]$. 
    Then, for any policy $\pi$, there exists a corresponding policy $\pi_K$ such that $\sup_{s \in \Scal, a \in \Acal} \pi_K(a \mid s) / \unif_\Acal(a) \leq K$, and
    \begin{equation*}
        |V(\pi) - V(\pi_K)| \leq 2 \sqrt{m}L H K^{-\alpha/m} \,.
    \end{equation*}
\end{lemma}
The proof of this lemma follows by comparing the value of $\pi$ with the value of a ``smoothed'' version of $\pi$ that has density ratio at most $K$, using the simulation lemma \citep{agarwal2019reinforcement}. We provide details in the appendix.

\pref{lem:policy-smoothing} implies that learning an approximately optimal policy under this density restriction gives an approximately optimal unrestricted policy when $K$ is sufficiently large, and when rewards and transitions are sufficiently smooth.
Similarly, for model-based algorithms under the same smoothness assumptions, if a learned model $\widehat\Mcal$ of the MDP gives $\epsilon$-accurate value estimates for all policies with density ratio at most $K$, then $\widehat\Mcal$ also gives $\epsilon + 2 L H K^{-\alpha/m}$-accurate estimates for all unrestricted policies.
In either case, this can be used to extend PAC bounds under the policy density restriction to unrestricted PAC bounds, by setting $K$ appropriately.
We will use the latter model-based approach for the reward-agnostic setting of FLAMBE in the following section.

\section{CASE STUDY: FLAMBE}
\label{sec:flambe}

We consider a case study on implementing our proposed strategies to extend PAC results for low-rank MDPs to continuous action spaces, by applying these to the FLAMBE algorithm of \citet{agarwal2020flambe}. We first present the FLAMBE algorithm along with the existing PAC result that holds for discrete actions, and then we extend this result to continuous actions under appropriate smoothness conditions. Then, at the end of the section we provide a brief discussion of implementation.

\subsection{FLAMBE Algorithm}

\begin{algorithm*}[t]\caption{FLAMBE: Feature Learning and Model-Based Exploration}\label{alg:flambe}
\begin{algorithmic}
\Require Environment $\Mcal$, function classes $\Phi, \Psi$, subroutines MLE and SAMP, parameters $\beta, n$.
\State Set $\rho_0$ to be the random policy, which takes all actions uniformly at random.
\State Set $D_h=\emptyset$ for each $h\in\{0,...,H-1\}$.
\For{$j=1,...,\jmax$}
    \For{$h=0,...,H-1$}
        \State Collect $n$ samples $(s_h, a_h, s_{h+1})$ by rolling into $s_h$ with $\rho_{j-1}$ and taking $a_h\sim \unif(\Acal)$.
        \State Add these samples to $D_h$.
        \State Solve maximum likelihood problem: $(\hat{\phi}_h, \hat{\psi}_h) \gets \text{MLE}(D_h)$.
        \State Set $\hat{T}_h(s_{h+1}\mid s_h, a_h)=\left\langle \hat{\phi}_h(s_h, a_h), \hat{\psi}_h(s_{h+1})\right\rangle$
    \EndFor
    \State For each $h$, call the planner (\pref{alg:planner}) with $h$ step model $\hat{T}_{0:h-1}$ and $\beta$ to obtain $\rho_h^{\text{pre}}$.
    \State Set $\rho_j=\unif\left(\{\rho_h^{\text{pre}} \circ \unif(\Acal) \}_{h=0}^{H-1}\right)$ to be uniform over the discovered $h$-step policies, 
    \State each of which we augment with random actions from time step $h$ onwards.
\EndFor
\end{algorithmic}
\end{algorithm*}

\begin{algorithm*}[t]\caption{Elliptical Planner}\label{alg:planner}
\begin{algorithmic}
\Require MDP $\widetilde{\Mcal}=\left(\phi_{0:\tilde{h}}, \psi_{0:\tilde{h}}\right)$, subroutine SAMP, parameter $\beta>0$. Initialize $\Sigma_0=I_{d\times d}$.
\For{$t=1,2,...,$}
    \State Compute $\pi_t$ such that
    \begin{align*}
    &\EE \left[\phi_{\tilde{h}}(s_{\tilde{h}}, a_{\tilde{h}})^T \Sigma_{t-1}^{-1} \phi_{\tilde{h}}(s_{\tilde{h}}, a_{\tilde{h}}) \mid \pi_t, \widetilde{\Mcal} \right]  \geq  \sup_{\pi} \EE \left[\phi_{\tilde{h}}(s_{\tilde{h}}, a_{\tilde{h}})^T \Sigma_{t-1}^{-1} \phi_{\tilde{h}}(s_{\tilde{h}}, a_{\tilde{h}}) \mid \pi, \widetilde{\Mcal} \right] - \frac{\beta}{2}
    \end{align*}
    \State If the objective is at most $\beta/2$, halt and output $\rho=\unif(\{\pi_\tau\}_{\tau<t})$.
    \State Compute $\Sigma_{\pi_t} = \EE \left[\phi_{\tilde{h}}(s_{\tilde{h}}, a_{\tilde{h}})
    \phi_{\tilde{h}}(s_{\tilde{h}}, a_{\tilde{h}})^T \mid \pi_t, \widetilde{\Mcal}
    \right]$. Update $\Sigma_t \gets \Sigma_{t-1}+\Sigma_{\pi_t}$.
\EndFor
\end{algorithmic}
\end{algorithm*}

We summarize the seminal FLAMBE algorithm of \citet{agarwal2020flambe} in \pref{alg:flambe}. This algorithm is computationally efficient assuming that we have access to two oracles: (1) a maximum likelihood estimation (MLE) oracle, which, given a dataset of $n$ tuples $\{(s_i,a_i,s_i') : i \in [n]\}$ of state, action, and successor state, returns the embeddings that maximize the empirical log likelihood $\sum_{i=1}^n \log( \phi(s_i,a_i)^\top \psi(s_i') )$; and (2) a sampling oracle, which, given a pair of embeddings $\phi \in \Phi$ and $\psi \in \Psi$, and a state-action pair $(s,a)$, can sample successor states $s' \sim \phi(s,a)^\top \psi(\cdot)$.
The MLE oracle is used to solve the maximum likelihood step of the algorithm, and the sampling oracle is used to solve the approximate policy optimization problem in the elliptical planner.\footnote{By \citet[Lemma 15]{agarwal2020flambe} this step can be solved in $\tilde\Ocal(d^3 H^6 (1/\beta)^4 \log(1/\delta))$ calls to the sampling oracle, with probability at least $1-\delta$.}
We note that for practical applications $(\Phi,\Psi)$ could be parameterized directly as a generative model, in which case the MLE oracle could be implemented using standard methods for fitting generative models, and the sampling oracle would be trivial.
In contrast to \citet{agarwal2020flambe}, our presentation of FLAMBE only performs approximate policy optimization in the planner. This approach aligns more closely with their analysis using the sampling oracle and does not impact the theoretical analysis.\footnote{This is because at the termination of the elliptical planner we still have $\EE [\phi_{\tilde{h}}(s_{\tilde{h}}, a_{\tilde{h}})^T \Sigma_{T}^{-1} \phi_{\tilde{h}}(s_{\tilde{h}}, a_{\tilde{h}}) \mid \pi, \widetilde{\Mcal} ] \leq \beta$ for every policy $\pi$, which is all that is required.}

Ultimately, FLAMBE returns a learnt model $\widehat\Mcal$ for $\Mcal$ and estimates $\hat\phi_h$, $\hat\psi_h$ for $\phi^*_h$ and $\psi^*_h$, respectively, for each $h \in [H]$. These learnt embeddings can then be used to perform planning for arbitrary reward functions, with the PAC guarantee for \emph{sparse} reward functions defined as follows:
\begin{definition}
\label{def:sparse-reward}
    Let $V(\pi;R,\Mcal)$ denote the value (expected sum of rewards) of policy $\pi$ under MDP model $\Mcal$, with reward function $R = (R_0,\ldots,R_{H-1})$ such that $R_h : \Scal \times \Acal \mapsto [0,1]$ for each $h \in [H]$. We say that $R$ is \emph{sparse} if $V(\pi;R,\Mcal) \leq 1$ for \emph{every} policy $\pi$ and model $\Mcal$.
\end{definition}
As a simple example, reward functions that can only take non-zero values at a single time step are sparse. 
The authors gave the following guarantee on FLAMBE's performance, which is adapted from \citet[Theorem 2 and Appendix B]{agarwal2020flambe}.
\begin{theorem}
\label{thm:flambe-guarantee}
For appropriate values of $\beta,J_{\max},n$,
FLAMBE returns a model $\widehat\Mcal$ that, with probability at least $1-\delta$, satisfies that for every policy $\pi$ and sparse reward function $R$,
\begin{equation*}
    \left| V(\pi;R,\widehat\Mcal) - V(\pi;R,\Mcal) \right| \leq \epsilon,
\end{equation*}
and the number of trajectories collected by the algorithm is
\begin{equation*}
    \tilde O \Bigg( H^{22} |\Acal|^9 d^7 \Big(\frac{1}{\epsilon}\Big)^{10} \ln\Big(\frac{|\Phi||\Psi|}{\delta}\Big) \Bigg)\,.
\end{equation*}
\end{theorem}
That is, with a polynomial sample complexity, the algorithm learns a model $\widehat\Mcal$ that can be used for policy optimization under \emph{any} sparse reward function. Although the assumption that $R$ is sparse may seem restrictive, as argued in \citet{agarwal2020flambe}, this is sufficient to ensure that we can fit $Q$ functions using the estimated embeddings $\hat\phi_h$ with error at most $\epsilon$.
We also note that this result can be easily extended to a similar PAC guarantee for \emph{all} (non-sparse) reward functions taking values in $[0,1]$, by replacing $\epsilon$ with $\epsilon/H$ in the sample complexity bound, since any reward function can be written as a sum of $H$ sparse reward functions. 

We also note that the exact same FLAMBE algorithm can be used with continuous action spaces as long as we can sample actions uniformly from $\Acal$: for example, when $\Acal = [0,1]^{m}$ we can sample uniformly random actions by sampling $m$ i.i.d. uniformly random variables from $[0,1]$. All of the results below apply to the same FLAMBE algorithm without any special modifications for continuous actions.

\subsection{Smoothness Assumption for Extension}

We now provide a concrete smoothness assumption that can be used to extend \pref{thm:flambe-guarantee} to continuous action spaces. As in \pref{sec:continuous-extension}, we consider continuous action spaces of the form $\mathcal{A}=[0,1]^{m}\subset \RR^{m}$ for some $m \in \NN^+$. 
These assumptions will allow us to use the methods detailed in \pref{sec:continuous-extension} in order to obtain generalized PAC bounds, which we present in \pref{sec:flambe-generalized-bounds}.

For this extension, we require that the embedding classes $\Phi$ and $\Psi$ satisfy a smoothness condition for transition function errors. Specifically, we require that the errors of transition functions allowed by $\Phi$ and $\Psi$ versus the true transition functions are smooth w.r.t. actions, as detailed below:

\begin{assum}[$\alpha$-smooth Transition Errors]
\label{assum:alpha-smooth-errors}
    There exists some error function $\Ecal : \Delta(\Scal) \times \Delta(\Scal) \to [0,1]$ on distributions of states, $\alpha_E \in (0,\infty)$, and $L_E < \infty$ such that the following hold: (1) $\textup{TV}(\rho,\rho') \leq \Ecal(\rho,\rho') \leq \Hcal(\rho,\rho')$ for all distributions $\rho,\rho' \in \Delta(\Scal)$, where $\textup{TV}(\cdot,\cdot)$ and $\Hcal(\cdot,\cdot)$ denote total variation distance and Hellinger distance, respectively; and (2) for any $s \in \Scal$, $\phi \in \Phi$, $\psi \in \Psi$, and $h \in [H]$, the function $a \mapsto \Ecal(\phi(s,a)^\top\psi(\cdot), \phi_h^*(s,a)^\top \psi_h^*(\cdot))$ is $\alpha_E$-smooth with $C^{\alpha_E}$-norm at most $L_E$.
\end{assum}
In other words, this assumption requires that the errors (in terms of some metric that lies between the total variation and Hellinger distances) of transition operators allowed by $\Phi,\Psi$ are smooth in the action $a$.

We note that this assumption is used to apply \pref{thm:is-alpha-smooth} in place of \pref{lem:is} in the FLAMBE analysis. While \citet{agarwal2020flambe} only use the total variation distance error function in their analysis, we decided to weaken this assumption and allow for various error functions. This relaxation is due to the fact that the smoothness assumption might be more feasible for certain error functions than others, depending on the context. We elaborate on two particular cases below.

\subsubsection*{Using Total Variation Distance}

In cases where the embeddings $\phi \in \Phi$ themselves are $\alpha_E$-H\"older continuous for some $\alpha_E \in (0,1]$, we can easily justify \pref{assum:alpha-smooth-errors} using total variation distance, as shown in the following lemma. Note that this case includes Lipschitz-continuous embeddings with $\alpha_E=1$.

\begin{lemma}
\label{lem:smooth-transitions}
    Suppose that $\phi(s,a)$ is $\alpha$-H\"older continuous in $a$ with norm $L_\Phi$ for every $\phi \in \Phi$ and $s \in \Scal$ (under $L_2$ norm) for some $\alpha \in (0,1]$, and that $\int \|\psi(s)\|_2 d\mu(s) \leq U$ for all $\psi \in \Psi$, for some fixed $U < \infty$.  Then, the conditions of \pref{assum:alpha-smooth-errors} hold with $\Ecal(\cdot,\cdot) = \textup{TV}(\cdot,\cdot)$, $\alpha_E = \alpha$, and $L_E = 2 U L_\Phi$.
\end{lemma}

We defer the proof of \pref{lem:smooth-transitions} to the appendix. Note that the assumption on $\Psi$ in \pref{lem:smooth-transitions} is slightly stronger than the assumed normalization condition in \pref{def:low-rank}, but it can also be easily justified in many settings; for example, if all embedding values are non-negative then it is easy to argue that $U \leq d$.\footnote{This holds since $\|\psi(s')\|_1 = \sum_{i=1}^d \psi(s')_i$ for non-negative embeddings, therefore we have $\int \|\psi(s)\|_2 d\mu(s) \leq \int \|\psi(s)\|_1 d\mu(s) = \| \int \psi(s) d\mu(s) \|_1 \leq d$, given the normalization in \pref{def:low-rank} and $\|\cdot\|_2 \leq \|\cdot\|_1 \leq \sqrt{d}\|\cdot\|_2 $.}

\subsubsection*{Using Hellinger Distance}

Unfortunately, \pref{lem:smooth-transitions} only allows us to leverage smoothness up to Lipschitz continuity ($\alpha_E = 1$). In order to exploit higher order smoothness, we can instead consider the Hellinger distance.
Recall that the Hellinger distance is defined according to $\Hcal(\rho,\rho') = \frac{1}{\sqrt{2}}\int (\sqrt{\rho(s)} - \sqrt{\rho'(s)})^2 d \mu(s)^{1/2}$, for distributions $\rho,\rho'$ that are continuous w.r.t. $\mu$. Since this corresponds to $L_2$ distance in the Hilbert space of square root densities, it is evident that for any integer $\alpha_E \in \NN^+$, \pref{assum:alpha-smooth-errors} holds with Hellinger distance as long as all order $\alpha_E$ mixed derivatives of $\sqrt{\phi(s,a)^\top \psi(\cdot)}$ w.r.t. $a$ exist, and that these derivatives have $L_2(\Scal)$ norm at most $L_E/2$.

Noting that $\frac{d}{dx} \sqrt{f(x)} = f(x)^{-1/2} \frac{d}{dx} f(x)$, one could easily justify this as long as $\phi(s,a)$ is uniformly $\alpha_E$-smooth in $a$ for all $\phi \in \Phi$ and $s \in \Scal$, and $\sup_{\phi,\psi,s,a} \int (\phi(s,a)^\top \psi(s'))^{1-2\alpha_E} d \mu(s') < \infty$. The latter of these conditions holds for any $\alpha_E$ if all transition operators allowed by $\Phi,\Psi$ have a uniformly bounded minimum density.
Similarly, \pref{assum:alpha-smooth-errors} may be justified for non-integer $\alpha_E > 1$ using analogous reasoning in terms of fractional derivatives.

Compared with using total variation distance, this approach allows us to leverage higher order smoothness of the embeddings $\phi(s,a)$, at the cost of requiring that the allowed transition densities $\phi(s,a)^\top \psi(\cdot)$ cannot be too concentrated around zero.

\subsection{Bound for Restricted Policies}
\label{sec:flambe-generalized-bounds}

We now present the first of our PAC results for FLAMBE in the continuous action setting under \pref{assum:alpha-smooth-errors}. For this result, we use both the methods described in \pref{sec:method-smoothness} and \pref{sec:method-boundedness}. In some parts of the FLAMBE analysis, \pref{lem:is} is applied to policies returned by the elliptical planner (which may have arbitrarily large density) and with the transition function error (which is smooth under \pref{assum:alpha-smooth-errors}) so we can apply the approach described in \pref{sec:method-smoothness}. In other places, \pref{lem:is} is applied to the policies that we want to bound estimation error for with discontinuous error functions, in which case we can apply the simple bound described in \pref{sec:method-boundedness}. With this, the sample complexity of our result depends on the density ratio $K$ of the policies being evaluated.

\begin{theorem}
\label{thm:continous-pi-guarantee}
Suppose $\Mcal$ has rank $d$, and let \pref{assum:realizability} and \pref{assum:alpha-smooth-errors} be given. Then, if we set the hyperparameters $\beta$, $n$, and $\jmax$ appropriately, with probability at least $1-\delta$, FLAMBE returns a model $\widehat\Mcal$ that satisfies
\begin{equation*}
    \left| V(\pi;R,\widehat\Mcal) - V(\pi;R,\Mcal) \right| \leq \epsilon
\end{equation*}
for every sparse reward function $R$, and every policy $\pi$ that is continuous with respect to the measure on $\Acal$ and satisfies  $\sup_{s, a} \pi(a \mid s) \leq K $.
The number of trajectories collected by the algorithm is
\begin{align*}
    \tilde O &\Bigg( H^{22+16\tau} K^{5+4\tau} d^{7+4\tau} \Big(\frac{1}{\epsilon}\Big)^{10+8\tau} \\
    & \quad \cdot L_E^{(9+ 8 \tau)\kappa} \ln\Big(\frac{|\Phi||\Psi|}{\delta}\Big) \Bigg)\,,
\end{align*}
where $\tau = m/\alpha_E$, and $\kappa = m/(m+\alpha_E)$.
\end{theorem}
The exact hyperparameter values that yield this result are given in the proof of this theorem (see \pref{app:proofs}). In this bound, $K$ corresponds to the maximum probability density (or sharpness) of policies we are guaranteed to accurately evaluate, and $\tau$ corresponds (inversely) to the assumed level of smoothness relative to the dimensionality of the action space. We note that as $\tau \to 0$ (\emph{i.e.}, we assume a very high order of smoothness relative to $m$) we recover the same dependence on $H$, $d$, and $1/\epsilon$ as in \pref{thm:flambe-guarantee} ($L_E$ and $K$ together play a similar role as $|\Acal|$, so comparing exponents of these parameters is less straightforward). However, when we have a lower level of smoothness, the polynomial factors in this bound can be significantly worse than in \pref{thm:flambe-guarantee}.
A simple intuition for this is that the smoother the transition operators allowed by the embedding classes $\Phi$ and $\Psi$, and the lower the number of degrees of freedom in the action space, the easier it is to explore $\Acal$.

\subsection{Bound for Unrestricted Policies}

Unfortunately, in many applications we would like to consider deterministic policies, for which $K=\infty$ and \pref{thm:continous-pi-guarantee} becomes vacuous. In order to address this, we take the approach described in \pref{sec:method-boundedness}. Here we require an additional assumption on the smoothness of the actual transition operator, as follows:

\begin{assum}($\alpha$-smooth Transitions)
\label{assum:alpha-smooth-transitions}
    There exists some fixed $\alpha_T \in (0,1]$ and $L_T < \infty$ such that $\textup{TV}(\phi_h^*(s,a)^\top \psi_h^*(\cdot), \phi_h^*(s,a')^\top \psi_h^*(\cdot)) \leq L_T \|a' - a\|^{\alpha_T}$, for every $h \in [H]$, $s \in \Scal$, and $a,a' \in \Acal$.
\end{assum}
This assumption is easy to justify as long as the true embeddings $\phi_h^*(s,a)$ are uniformly $\alpha_T$-H\"older continuous for all $s,h$ by applying \pref{lem:smooth-transitions}.
Given this, we are ready to provide our second PAC result.

\begin{theorem}
\label{thm:continous-r-guarantee}
Let the conditions and definitions of \pref{thm:continous-pi-guarantee} be given, as well as \pref{assum:alpha-smooth-transitions}. Then, if we set the hyperparameters $\beta$, $n$, and $\jmax$ appropriately, with probability at least $1-\delta$, FLAMBE returns a model $\widehat\Mcal$ that satisfies
\begin{equation*}
    \left| V(\pi;R,\widehat\Mcal) - V(\pi;R,\Mcal) \right| \leq \epsilon
\end{equation*}
for every policy $\pi$, and every sparse reward function $R$ such that $R_h(s,a)$ is $\alpha_R$-H\"older continuous with norm at most $L_R$ for all $h \in [H]$ and $s \in \Scal$. The number of trajectories collected by the algorithm is 
\begin{align*}
    \tilde O &\Bigg( H^{22+16\tau+(4 \tau + 5)\sigma} d^{7+4\tau} \Big(\frac{1}{\epsilon}\Big)^{10+8\tau+(4 \tau + 5)\sigma} \\
    & \quad \cdot L_E^{(9 + 8 \tau)\kappa} L^{(4 \tau + 5)\sigma} \ln\Big(\frac{|\Phi||\Psi|}{\delta}\Big) \Bigg)\,,
\end{align*}
where $\sigma = m / \min(\alpha_T, \alpha_R)$, and $L = \max(L_T, L_R)$.
\end{theorem}

We note that this theorem similarly gives us a sample complexity that is polynomial in all problem variables for fixed action space dimension and order of smoothness. We can see that the less smooth the reward functions under consideration (\emph{i.e.}, the smaller $\alpha_R$ is and the greater $L_R$ is), the greater the number of trajectories needed to learn a model that can accurately estimate policy values. Similarly, the less smooth the transition function is, the more trajectories are needed.
Note as well that, since $\alpha_T \leq 1$, we cannot have $\sigma \to 0$ for very smooth problem instances, so this bound gives worse exponents than \pref{thm:continous-pi-guarantee} for smooth problem instances, at the benefit of allowing for arbitrarily sharp policies.
The proof of \pref{thm:continous-r-guarantee} is a fairly straightforward application of \pref{lem:policy-smoothing} to \pref{thm:continous-pi-guarantee}, along the lines discussed in \pref{sec:method-boundedness}, and we defer details to the appendix.

\subsection{Discussion of Implementation}

FLAMBE with continuous actions can mostly be implemented identically as with discrete actions. The only slight complication comes with implementing the elliptical planner. As in the discrete case, we can implement this planner using backward dynamic programming, since under MDP linearity $Q$ functions are linear in $\phi(s,a)$ \citep{jin2020provably}. However, implementing this requires computing $\max_a w^\top \hat \phi(s,a)$ for various vectors $w$, states $s$, and estimated embeddings $\hat\phi$. With discrete actions this optimization is trivial, whereas when $\Acal$ is continuous we require some kind of non-trivial optimization procedure. Below we discuss some possible approaches.

\paragraph{Using Concave Embeddings} If we ensure that by design $\phi(s,a)$ is concave in $a$ for all $\phi \in \Phi$, then the optimization problem $\max_a w^\top \hat \phi(s,a)$ is tractable (equivalent to minimizing a convex function) as long as $w$ is non-negative. This latter condition may be additional assumptions on a problem by problem basis depending on the embedding structure.

\paragraph{Using Grid Search} We could approximately solve the problem $\max_a w^\top \hat \phi(s,a)$ by performing grid search over $a \in [0,1]^m$. When the embeddings $\hat\phi(s,a)$ are relatively smooth in $a$, and $m$ is not too large, then this approach should be reasonably tractable.

\paragraph{Using Blackbox Non-convex optimiztion} A third approach, which might be the most useful in many problems in practice, is to simply treat $\max_a w^\top \hat \phi(s,a)$ as a blackbox optimization problem. For example, for some fixed iteration of dynamic programming, where $w$ and $\hat\phi$ are fixed, we could learn a mapping from $s$ to $\argmax_a w^\top \hat\phi(s,a)$ using machine learning procedures such as deep learning, treating this as a general nonconvex policy optimization problem. While this kind of approach may be difficult to theoretically analyze, it may be effective in ensuring accurrate elliptical planning performance in practice.

\section{Conclusion}
\label{sec:conclusion}

We studied the problem of generalizing low-rank MDP methods to settings with continuous action spaces. We first provided a detailed discussion of exactly how/why low-rank MDP PAC bounds depend on the size of the action space, and we discussed multiple techniques that could be used to extend existing PAC results to continuous action spaces by utilizing smoothness of problems parameters w.r.t. actions. Then, as a case study, we considered the seminal reward-agnostic FLAMBE algorithm of \citet{agarwal2020flambe}, and investigated how to extend their theory to continuous action spaces using these methods. Based on these insights, we provided novel PAC bounds for low-rank MDPs with continuous action spaces, and we derived sample complexities that are polynomial in all problem parameters except for the assumed degree of smoothness relative to the dimension of the action space. We hope that this work is useful and impactful in motivating further research on low-rank MDPs, as well as other practical PAC RL settings, with continuous action spaces. In particular, an interesting and relevant avenue of future research would be to construct lower bounds on sample complexity under appropriate smoothness assumptions, further advancing the theoretical foundations and algorithmic approaches in these areas.

\subsubsection*{Acknowledgements}
We thank the anonymous reviewers for their helpful comments. This material is based upon work supported by the National Science Foundation under Grant No. 1846210 and the U.S. Department of Energy, Office of Science, Office of Advanced Scientific Computing Research, under Award
Number DE-SC0023112.

\newpage

\bibliography{ref}

\begin{thebibliography}{40}
\providecommand{\natexlab}[1]{#1}
\providecommand{\url}[1]{\texttt{#1}}
\expandafter\ifx\csname urlstyle\endcsname\relax
  \providecommand{\doi}[1]{doi: #1}\else
  \providecommand{\doi}{doi: \begingroup \urlstyle{rm}\Url}\fi

\bibitem[Agarwal and Zhang(2022)]{agarwal2022non}
A.~Agarwal and T.~Zhang.
\newblock Non-linear reinforcement learning in large action spaces: Structural conditions and sample-efficiency of posterior sampling.
\newblock In \emph{Conference on Learning Theory}, pages 2776--2814. PMLR, 2022.

\bibitem[Agarwal et~al.(2019)Agarwal, Jiang, and Kakade]{agarwal2019reinforcement}
A.~Agarwal, N.~Jiang, and S.~M. Kakade.
\newblock Reinforcement learning: Theory and algorithms.
\newblock 2019.

\bibitem[Agarwal et~al.(2020)Agarwal, Kakade, Krishnamurthy, and Sun]{agarwal2020flambe}
A.~Agarwal, S.~Kakade, A.~Krishnamurthy, and W.~Sun.
\newblock Flambe: Structural complexity and representation learning of low rank mdps.
\newblock In \emph{Advances in Neural Information Processing Systems}, volume~33, pages 20095--20107, 2020.

\bibitem[Behzadan and Holst(2021)]{behzadan2021multiplication}
A.~Behzadan and M.~Holst.
\newblock Multiplication in sobolev spaces, revisited.
\newblock \emph{Arkiv f{\"o}r Matematik}, 59\penalty0 (2):\penalty0 275--306, 2021.

\bibitem[Brezis and Mironescu(2018)]{brezis2018gagliardo}
H.~Brezis and P.~Mironescu.
\newblock Gagliardo--nirenberg inequalities and non-inequalities: the full story.
\newblock In \emph{Annales de l'Institut Henri Poincar{\'e} C, Analyse non lin{\'e}aire}, volume~35, pages 1355--1376. Elsevier, 2018.

\bibitem[Brezis and Mironescu(2019)]{brezis2019sobolev}
H.~Brezis and P.~Mironescu.
\newblock Where sobolev interacts with gagliardo--nirenberg.
\newblock \emph{Journal of functional analysis}, 277\penalty0 (8):\penalty0 2839--2864, 2019.

\bibitem[Cai et~al.(2017)Cai, Ren, Zhang, Malialis, Wang, Yu, and Guo]{cai2017real}
H.~Cai, K.~Ren, W.~Zhang, K.~Malialis, J.~Wang, Y.~Yu, and D.~Guo.
\newblock Real-time bidding by reinforcement learning in display advertising.
\newblock In \emph{Proceedings of the Tenth ACM International Conference on Web Search and Data Mining}, pages 661--670, 2017.

\bibitem[Chen et~al.(2022)Chen, Modi, Krishnamurthy, Jiang, and Agarwal]{chen2022statistical}
J.~Chen, A.~Modi, A.~Krishnamurthy, N.~Jiang, and A.~Agarwal.
\newblock On the statistical efficiency of reward-free exploration in non-linear rl.
\newblock \emph{Advances in Neural Information Processing Systems}, 35:\penalty0 20960--20973, 2022.

\bibitem[Cheng et~al.(2023)Cheng, Huang, Yang, and Liang]{cheng2023improved}
Y.~Cheng, R.~Huang, J.~Yang, and Y.~Liang.
\newblock Improved sample complexity for reward-free reinforcement learning under low-rank mdps.
\newblock \emph{arXiv preprint arXiv:2303.10859}, 2023.

\bibitem[Chinchali et~al.(2018)Chinchali, Hu, Chu, Sharma, Bansal, Misra, Pavone, and Katti]{chinchali2018cellular}
S.~Chinchali, P.~Hu, T.~Chu, M.~Sharma, M.~Bansal, R.~Misra, M.~Pavone, and S.~Katti.
\newblock Cellular network traffic scheduling with deep reinforcement learning.
\newblock In \emph{Thirty-second AAAI conference on artificial intelligence}, 2018.

\bibitem[Domingues et~al.(2020)Domingues, M{\'e}nard, Pirotta, Kaufmann, and Valko]{domingues2020regret}
O.~D. Domingues, P.~M{\'e}nard, M.~Pirotta, E.~Kaufmann, and M.~Valko.
\newblock Regret bounds for kernel-based reinforcement learning.
\newblock \emph{arXiv preprint arXiv:2004.05599}, 2020.

\bibitem[Du et~al.(2019)Du, Krishnamurthy, Jiang, Agarwal, Dudik, and Langford]{du2019provably}
S.~Du, A.~Krishnamurthy, N.~Jiang, A.~Agarwal, M.~Dudik, and J.~Langford.
\newblock Provably efficient rl with rich observations via latent state decoding.
\newblock In \emph{International Conference on Machine Learning}, pages 1665--1674. PMLR, 2019.

\bibitem[Du et~al.(2021)Du, Kakade, Lee, Lovett, Mahajan, Sun, and Wang]{du2021bilinear}
S.~S. Du, S.~M. Kakade, J.~D. Lee, S.~Lovett, G.~Mahajan, W.~Sun, and R.~Wang.
\newblock Bilinear classes: A structural framework for provable generalization in rl.
\newblock \emph{arXiv preprint arXiv:2103.10897}, 2021.

\bibitem[Fiorenza et~al.(2021)Fiorenza, Formica, Roskovec, and Soudsk{\`y}]{fiorenza2021detailed}
A.~Fiorenza, M.~R. Formica, T.~G. Roskovec, and F.~Soudsk{\`y}.
\newblock Detailed proof of classical gagliardo--nirenberg interpolation inequality with historical remarks.
\newblock \emph{Zeitschrift f{\"u}r Analysis und ihre Anwendungen}, 40\penalty0 (2):\penalty0 217--236, 2021.

\bibitem[Gin{\'e} and Nickl(2021)]{gine2021mathematical}
E.~Gin{\'e} and R.~Nickl.
\newblock \emph{Mathematical foundations of infinite-dimensional statistical models}.
\newblock Cambridge university press, 2021.

\bibitem[Hu et~al.(2022)Hu, Kallus, and Mao]{hu2022smooth}
Y.~Hu, N.~Kallus, and X.~Mao.
\newblock Smooth contextual bandits: Bridging the parametric and nondifferentiable regret regimes.
\newblock \emph{Operations Research}, 70\penalty0 (6):\penalty0 3261--3281, 2022.

\bibitem[Jia et~al.(2023)Jia, Xie, Kallus, and Frazier]{jia2023smooth}
S.~Jia, Q.~Xie, N.~Kallus, and P.~I. Frazier.
\newblock Smooth non-stationary bandits.
\newblock \emph{arXiv preprint arXiv:2301.12366}, 2023.

\bibitem[Jiang et~al.(2017)Jiang, Krishnamurthy, Agarwal, Langford, and Schapire]{jiang2017contextual}
N.~Jiang, A.~Krishnamurthy, A.~Agarwal, J.~Langford, and R.~E. Schapire.
\newblock Contextual decision processes with low bellman rank are pac-learnable.
\newblock In \emph{International Conference on Machine Learning}, pages 1704--1713. PMLR, 2017.

\bibitem[Jin et~al.(2020)Jin, Yang, Wang, and Jordan]{jin2020provably}
C.~Jin, Z.~Yang, Z.~Wang, and M.~I. Jordan.
\newblock Provably efficient reinforcement learning with linear function approximation.
\newblock In \emph{Conference on Learning Theory}, pages 2137--2143. PMLR, 2020.

\bibitem[Kober et~al.(2013)Kober, Bagnell, and Peters]{kober2013reinforcement}
J.~Kober, J.~A. Bagnell, and J.~Peters.
\newblock Reinforcement learning in robotics: A survey.
\newblock \emph{The International Journal of Robotics Research}, 32\penalty0 (11):\penalty0 1238--1274, 2013.

\bibitem[Krishnamurthy et~al.(2016)Krishnamurthy, Agarwal, and Langford]{krishnamurthy2016pac}
A.~Krishnamurthy, A.~Agarwal, and J.~Langford.
\newblock Pac reinforcement learning with rich observations.
\newblock \emph{arXiv preprint arXiv:1602.02722}, 2016.

\bibitem[Leoni(2017)]{leoni2017first}
G.~Leoni.
\newblock \emph{A first course in Sobolev spaces}.
\newblock American Mathematical Soc., 2017.

\bibitem[Li et~al.(2019)Li, Zheng, and Zheng]{li2019deep}
Y.~Li, W.~Zheng, and Z.~Zheng.
\newblock Deep robust reinforcement learning for practical algorithmic trading.
\newblock \emph{IEEE Access}, 7:\penalty0 108014--108022, 2019.

\bibitem[Liu and Wang(2017)]{liu2017best}
J.-G. Liu and J.~Wang.
\newblock On the best constant for gagliardo-nirenberg interpolation inequalities.
\newblock \emph{arXiv preprint arXiv:1712.10208}, 2017.

\bibitem[Liu et~al.(2021)Liu, Wang, and Singh]{liu2021smooth}
Y.~Liu, Y.~Wang, and A.~Singh.
\newblock Smooth bandit optimization: generalization to holder space.
\newblock In \emph{International Conference on Artificial Intelligence and Statistics}, pages 2206--2214. PMLR, 2021.

\bibitem[Lykouris and Vassilvtiskii(2018)]{lykouris2018competitive}
T.~Lykouris and S.~Vassilvtiskii.
\newblock Competitive caching with machine learned advice.
\newblock In \emph{International Conference on Machine Learning}, pages 3296--3305. PMLR, 2018.

\bibitem[Misra et~al.(2020)Misra, Henaff, Krishnamurthy, and Langford]{misra2020kinematic}
D.~Misra, M.~Henaff, A.~Krishnamurthy, and J.~Langford.
\newblock Kinematic state abstraction and provably efficient rich-observation reinforcement learning.
\newblock In \emph{International conference on machine learning}, pages 6961--6971. PMLR, 2020.

\bibitem[Modi et~al.(2021)Modi, Chen, Krishnamurthy, Jiang, and Agarwal]{modi2021model}
A.~Modi, J.~Chen, A.~Krishnamurthy, N.~Jiang, and A.~Agarwal.
\newblock Model-free representation learning and exploration in low-rank mdps.
\newblock \emph{arXiv preprint arXiv:2102.07035}, 2021.

\bibitem[Sallab et~al.(2017)Sallab, Abdou, Perot, and Yogamani]{sallab2017deep}
A.~E. Sallab, M.~Abdou, E.~Perot, and S.~Yogamani.
\newblock Deep reinforcement learning framework for autonomous driving.
\newblock \emph{Electronic Imaging}, 2017\penalty0 (19):\penalty0 70--76, 2017.

\bibitem[Sinclair et~al.(2021)Sinclair, Banerjee, and Yu]{sinclair2021adaptive}
S.~R. Sinclair, S.~Banerjee, and C.~L. Yu.
\newblock Adaptive discretization in online reinforcement learning.
\newblock \emph{arXiv preprint arXiv:2110.15843}, 2021.

\bibitem[Song and Sun(2019)]{song2019efficient}
Z.~Song and W.~Sun.
\newblock Efficient model-free reinforcement learning in metric spaces.
\newblock \emph{arXiv preprint arXiv:1905.00475}, 2019.

\bibitem[Stein(1970)]{stein1970singular}
E.~M. Stein.
\newblock \emph{Singular integrals and differentiability properties of functions}.
\newblock Princeton university press, 1970.

\bibitem[Stone(1982)]{stone1982optimal}
C.~J. Stone.
\newblock Optimal global rates of convergence for nonparametric regression.
\newblock \emph{The annals of statistics}, pages 1040--1053, 1982.

\bibitem[Sun et~al.(2019)Sun, Jiang, Krishnamurthy, Agarwal, and Langford]{sun2019model}
W.~Sun, N.~Jiang, A.~Krishnamurthy, A.~Agarwal, and J.~Langford.
\newblock Model-based rl in contextual decision processes: Pac bounds and exponential improvements over model-free approaches.
\newblock In \emph{Conference on learning theory}, pages 2898--2933. PMLR, 2019.

\bibitem[Tessler et~al.(2021)Tessler, Shpigelman, Dalal, Mandelbaum, Kazakov, Fuhrer, Chechik, and Mannor]{tessler2021reinforcement}
C.~Tessler, Y.~Shpigelman, G.~Dalal, A.~Mandelbaum, D.~H. Kazakov, B.~Fuhrer, G.~Chechik, and S.~Mannor.
\newblock Reinforcement learning for datacenter congestion control.
\newblock \emph{arXiv preprint arXiv:2102.09337}, 2021.

\bibitem[Tsybakov(2008)]{tsybakov2008}
A.~B. Tsybakov.
\newblock \emph{Introduction to Nonparametric Estimation}.
\newblock Springer Publishing Company, Incorporated, 1st edition, 2008.
\newblock ISBN 0387790519.

\bibitem[Uehara et~al.(2021)Uehara, Zhang, and Sun]{uehara2021representation}
M.~Uehara, X.~Zhang, and W.~Sun.
\newblock Representation learning for online and offline rl in low-rank mdps.
\newblock \emph{arXiv preprint arXiv:2110.04652}, 2021.

\bibitem[Yang and Wang(2019)]{yang2019sample}
L.~Yang and M.~Wang.
\newblock Sample-optimal parametric q-learning using linearly additive features.
\newblock In \emph{International Conference on Machine Learning}, pages 6995--7004. PMLR, 2019.

\bibitem[Yang and Wang(2020)]{yang2020reinforcement}
L.~Yang and M.~Wang.
\newblock Reinforcement learning in feature space: Matrix bandit, kernels, and regret bound.
\newblock In \emph{International Conference on Machine Learning}, pages 10746--10756. PMLR, 2020.

\bibitem[Zheng et~al.(2018)Zheng, Zhang, Zheng, Xiang, Yuan, Xie, and Li]{zheng2018drn}
G.~Zheng, F.~Zhang, Z.~Zheng, Y.~Xiang, N.~J. Yuan, X.~Xie, and Z.~Li.
\newblock Drn: A deep reinforcement learning framework for news recommendation.
\newblock In \emph{Proceedings of the 2018 World Wide Web Conference}, pages 167--176, 2018.

\end{thebibliography}
\bibliographystyle{abbrvnat}

\section*{Checklist}

 \begin{enumerate}

 \item For all models and algorithms presented, check if you include:
 \begin{enumerate}
   \item A clear description of the mathematical setting, assumptions, algorithm, and/or model. \textbf{Yes}
   \item An analysis of the properties and complexity (time, space, sample size) of any algorithm. \textbf{Yes}
   \item (Optional) Anonymized source code, with specification of all dependencies, including external libraries. \textbf{Not Applicable}
 \end{enumerate}

 \item For any theoretical claim, check if you include:
 \begin{enumerate}
   \item Statements of the full set of assumptions of all theoretical results. \textbf{Yes}
   \item Complete proofs of all theoretical results. \textbf{Yes}
   \item Clear explanations of any assumptions. \textbf{Yes}    
 \end{enumerate}

 \item For all figures and tables that present empirical results, check if you include:
 \begin{enumerate}
   \item The code, data, and instructions needed to reproduce the main experimental results (either in the supplemental material or as a URL). \textbf{Not Applicable}
   \item All the training details (e.g., data splits, hyperparameters, how they were chosen). \textbf{Not Applicable}
         \item A clear definition of the specific measure or statistics and error bars (e.g., with respect to the random seed after running experiments multiple times). \textbf{Not Applicable}
         \item A description of the computing infrastructure used. (e.g., type of GPUs, internal cluster, or cloud provider). \textbf{Not Applicable}
 \end{enumerate}

 \item If you are using existing assets (e.g., code, data, models) or curating/releasing new assets, check if you include:
 \begin{enumerate}
   \item Citations of the creator If your work uses existing assets. \textbf{Not Applicable}
   \item The license information of the assets, if applicable. \textbf{Not Applicable}
   \item New assets either in the supplemental material or as a URL, if applicable. \textbf{Not Applicable}
   \item Information about consent from data providers/curators. \textbf{Not Applicable}
   \item Discussion of sensible content if applicable, e.g., personally identifiable information or offensive content. \textbf{Not Applicable}
 \end{enumerate}

 \item If you used crowdsourcing or conducted research with human subjects, check if you include:
 \begin{enumerate}
   \item The full text of instructions given to participants and screenshots. \textbf{Not Applicable}
   \item Descriptions of potential participant risks, with links to Institutional Review Board (IRB) approvals if applicable. \textbf{Not Applicable}
   \item The estimated hourly wage paid to participants and the total amount spent on participant compensation. \textbf{Not Applicable}
 \end{enumerate}

 \end{enumerate}

\appendix
\newpage
\onecolumn

\section{PROOFS} \label{app:proofs}
\allowdisplaybreaks

\subsection{\pfref{thm:is-alpha-smooth}}\label{sec:is-alpha-smooth-proof}

The proof is based on a recent contribution (\pref{lem:gn-ineq-f}) from the field of functional analysis. Before delving into it, we introduce a prior, more established lemma to develop a better intuition about the final result. 

\begin{lemma}[Gagliardo–Nirenberg interpolation inequality]\label{lem:gn-ineq} Let $1\leq p, q\leq \infty$, $s\in \NN, k\in \NN^+$ such that $0\leq k<s$ and let $\theta$ be such that
\begin{align*}
    0\leq \theta\leq 1-k/s
\end{align*}
and 
\begin{align*}
    (1-\theta)\left(\frac{1}{p}-\frac{s-k}{m}\right) + \theta\left(\frac{1}{q}+\frac{k}{m}\right)=\frac{1}{r}\in (-\infty, 1]
\end{align*}
Then there exists a constant $C_{s,m,k,p,q,r}>0$ that depends only on $s,m,k,p,q,r$ such that:
\begin{align*}
    \|\nabla^k f\|_{L^r(\RR^{m})} \leq C_{s,m,k,p,q,r}\|\nabla^s f\|_{L^p(\RR^{m})}^{1-\theta}\|f\|_{L^q(\RR^{m})}^\theta
\end{align*}
for every $f:\RR^{m}\rightarrow\RR$ with the following properties: (1) $f$ is $s$ times differentiable; (2) $\|\nabla^s f\|_{L^p(\RR^{m})}$, $\|f\|_{L^q(\RR^{m})}<\infty$ ; and (3) $f$ vanishes at infinity.     
\end{lemma}

The Gagliardo–Nirenberg (GN) interpolation inequality was originally proposed independently by Emilio Gagliardo and Louis Nirenberg in 1958, but the generality of the statement, as well as the details of the proof have been subject to refinement until recently. The version presented in \pref{lem:gn-ineq} is based on Theorem 12.87 in \citet{leoni2017first}. The GN interpolation inequality has found numerous applications in various branches of mathematics, including the theory of elliptic and parabolic partial differential equations, harmonic analysis, and functional analysis.

The GN inequality will be central to our proof of \pref{thm:is-alpha-smooth}. To see this, let $g:\RR^{m}\rightarrow\RR$ be a function defined on all of $\RR^{m}$ that is $\alpha$-smooth with $\alpha\geq 1$ and vanishes at $\infty$. Setting $p, r=\infty$, $s=\lfloor\alpha\rfloor$, $k=0$, $q=1$ in \pref{lem:gn-ineq}, we immediately obtain $\theta=\frac{\lfloor\alpha\rfloor}{m+\lfloor\alpha\rfloor}$ and a bound of the form:
\begin{align}\label{eq:GN-app}
    \|g\|_\infty \leq C_{\lfloor\alpha\rfloor, m, 0, \infty, 1, \infty} L^{\frac{m}{m+\lfloor\alpha\rfloor}}\|g\|_{L^1(\RR^{m})}^{\frac{\lfloor\alpha\rfloor}{m+\lfloor\alpha\rfloor}}
\end{align}
\begin{remark}[GN Inequality Constant]
    In \pref{eq:GN-app}, the GN constant $C_{\lfloor\alpha\rfloor, m, 0, \infty, 1, \infty}$ depends on $p, r=\infty$ which can potentially lead to a large constant value. Finding the best values for the GN constant is an area of active research with several recent results finding tight values for these constants in specific scenarios. For example, \citet{liu2017best} find that for $s=1$ (i.e. $\lfloor\alpha\rfloor=1$) the constant is $C_{1, m, 0, \infty, 1, \infty}=1$ (see Eq. 1.20 from their paper). We assume that the GN constant for $s>1$ will similarly have values that are not too large. 
\end{remark}

Two challenges arise when applying \pref{lem:gn-ineq} directly to our scenario: (1) the GN inequality applies to functions over the entire $\RR^{m}$, while our functions are defined on a bounded domain, and (2) the bound lacks precision for non-integer $\alpha$ values, requiring an adjustment similar to \pref{thm:is-alpha-smooth}. Traditionally, the first issue has been mitigated using linear extension operators for domains with Lipschitz boundaries (see \citet{stein1970singular}). The second challenge has been fully resolved only recently, through the work of Ha"im Brezis and Petru Mironescu \citep{brezis2018gagliardo, brezis2019sobolev}. This result relies on the Sobolev norm, which we define here for $\alpha$-smooth functions:

\begin{definition}[Sobolev norm of $\alpha$-smooth functions]
    For two integers $s\in \NN$ and $p\in \NN^+$, the Sobolev norm $\|\cdot\|_{W^{s,p}}$ of an $s$-smooth function defined on $\Dcal$ is defined as:
    \begin{align*}
        \|f\|_{W^{s,p}(\Dcal)}=\sum_{j=0}^s \|\nabla^j f\|_{L^p(\Dcal)}
    \end{align*}
    For $\alpha$-smooth functions with a potentially non-integer $\alpha\in(0, \infty)$, the fractional Sobolev norm for the $(\alpha, \infty)$ indices is given by:
    \begin{align*}
        & \|f\|_{W^{\alpha,\infty}(\Dcal)}=\|f\|_{W^{\lfloor \alpha\rfloor,\infty}(\Dcal)}+\max_{\beta:|\beta|=\lfloor \alpha\rfloor}\left|D^\beta f\right|_{C^{0, \{\alpha\}}}
    \end{align*}
\end{definition}

We now present the GN inequality for fractional smoothness from \citet{brezis2019sobolev}.

\begin{lemma}\label{lem:gn-ineq-f}[Fractional GN inequality, Theorem 1 in \citet{brezis2019sobolev}]
    Let $\Dcal$ be $\RR^{m}$, a half space or a Lipschitz bounded domain in $\RR^{m}$, and let $s_1,s_2,r,p_1, p_2, q, \theta, m$ satisfy:
    \begin{align*}
        & s_1,s_2,r\in \RR^+, p_1, p_2, p\in \NN, 1\leq p_1, p_2, p\leq \infty,\\
        & r<s:=\theta s_1+(1-\theta)s_2, \text{ and } \\
        & \frac{1}{q} = \left(\frac{\theta}{p_1}+\frac{1-\theta}{p_2}\right) - \frac{s-r}{m}.
    \end{align*}
    Then, the following inequality holds 
    \begin{align*}
        \|f\|_{W^{r,q}(\Dcal)}\leq C_{m, r, q, s_1,s_2,p_1,p_2} \|f\|^\theta_{W^{s_1,p_1}(\Dcal)}\|f\|^{1-\theta}_{W^{s_2,p_2}(\Dcal)}
    \end{align*}
    for every $f:\Dcal\rightarrow\RR$ with $\|f\|_{W^{s_1,p_1}(\Dcal)}, \|f\|_{W^{s_2,p_2}(\Dcal)}<\infty$ if and only if at least one of the following statements is false:
    $
    \begin{cases}
        s_2\in \NN \text{ and } s_2\geq 1, \\
        p_2=1, \\
        0<s_2-s_1\leq 1-\frac{1}{p_1}. 
    \end{cases}
    $.\\
    Here $C_{m, r, q, s_1,s_2,p_1,p_2}$ is a constant that depends only on $m, r, q, s_1,s_2,p_1,p_2$ and $\|f\|_{W^{s,p}(\Dcal)}$ is the fractional Sobolev norm.
\end{lemma}

We are now ready to provide a proof for \pref{thm:is-alpha-smooth}. 

\begin{proof}[\pfref{thm:is-alpha-smooth}] Let $g(a):=\EE_{s\sim \rho}[f(s,a)]$. Since $f(s,a)$ is $\alpha$-smooth in the action $a$ with $\alpha\in(0, \infty)$ and $\|f\|_{C^\alpha}\leq L$, it immediately follows that $g(a)$ is also $\alpha$-smooth with $\|g\|_{C^\alpha}\leq L$. Furthermore, since $\Acal$ is a Lipschitz domain, we can apply \pref{lem:gn-ineq-f} directly by setting the following parameters: $q=\infty, r=0, p_1=1, s_1=0, p_2=\infty, s_2=\alpha$. This yields a coefficient $\theta=\frac{\alpha}{m+\alpha}$ and 
\begin{align}\label{eq:bound-proof-1}
    \|g\|_{W^{0,\infty}(\Acal)}\leq C_{m, 0, \infty, 0,\alpha,1,\infty} \|g\|^{\frac{\alpha}{m+\alpha}}_{W^{0,1}(\Acal)}\|g\|^{\frac{m}{m+\alpha}}_{W^{\alpha,\infty}(\Acal)}
\end{align}
We note that the conditions $\|g\|_{W^{s_1, p_1}(\Acal)}=\|g\|_{W^{0, 1}(\Acal)}<\infty$, $\|g\|_{W^{s_2, p_2}(\Acal)}=\|g\|_{W^{\alpha, \infty}(\Acal)}<\infty$ hold from the properties of $\alpha$-smooth functions (\pref{def:a-smooth}). In particular, $\|g\|_{W^{\alpha, \infty}(\Acal)}=\sum_{j=0}^{\lfloor\alpha\rfloor} \|\nabla^j g\|_{L^p(\Dcal)}+\max_{\beta:|\beta|=\lfloor \alpha\rfloor}\left|D^\beta g\right|_{C^{0, \{\alpha\}}}\leq (\lfloor\alpha\rfloor+2)L$. Absorbing all constants into one constant $C$ and rearranging the terms, we can rewrite \pref{eq:bound-proof-1} as:
\begin{align*}
    \|g\|_\infty\leq C L^{\frac{\alpha}{m+\alpha}}\|g\|^{\frac{\alpha}{m+\alpha}}_{L^1(\Acal)}
\end{align*}
We further write $\|g\|_{L^1(\Acal)}=\int_\Acal g(a)d\mu(a)=\mu(\Acal)\int_\Acal g(a)\frac{1}{\mu(\Acal)}d\mu(a)=\EE_{a\sim \text{unif}_{\Acal}}g(a)$ since $\mu(\Acal)=1$. Putting everything together, we obtain following inequality:
\begin{align*}
     \sup_{a \in \Acal} g(a) \leq C L^{\frac{m}{m+\alpha}} \left[\EE_{a \sim \unif_\Acal} g(a)\right]^\frac{\alpha}{m+\alpha} 
\end{align*}
The desired conclusion immediately follows by noting that $\EE_{s\in\rho, a\sim \pi(\cdot\mid s)}[f(s,a)]\leq \EE_{s\sim \rho}[\sup_{a\in \Acal}f(s,a)]=\sup_{a\in \Acal}g(a)$ and $\EE_{s\sim\rho, a\sim \unif_\Acal}[f(s,a)]=\EE_{a\sim \unif_\Acal}\EE_{s\sim\rho}[f(s,a)]=\EE_{a\sim \unif_\Acal}g(a)$. This concludes our proof.
\end{proof}

\begin{remark}[{Generality of $\Acal=\left[0,1\right]^n$}]
    We note that the result in \pref{thm:is-alpha-smooth} is scale-invariant. That is, if we consider $\Acal=\left[0,1\right]^n$ to be the scaling of a larger space $\Acal'=\left[0,A_{max}\right]^n$ with $f':\Acal'\rightarrow \RR$ such that $f'(a')=f(a'/A_{max})$, the bound in \pref{thm:is-alpha-smooth} is the same across both $\Acal$ and $\Acal'$. Intuitively, this is because the bound on the derivatives increases as the space shrinks, but the measure of the space decreases proportionally. Let $L'$ be bound in \pref{def:a-smooth} of $f'$. Then $L=\frac{L'}{A_{max}^\alpha}$. Noting that $\sup_{a \in \Acal} f(a)=\sup_{a' \in \Acal'} f(a')$, we have:
    \begin{align*}
        \sup_{a \in \Acal} f(a) &\leq C L^{\frac{m}{m+\alpha}} \left[\EE_{a \sim \unif_\Acal} f(a)\right]^\frac{\alpha}{m+\alpha}\\
        \Rightarrow \sup_{a' \in \Acal'} f(a') & \leq C \left(\frac{L'}{A_{max}^\alpha}\right)^{\frac{m}{m+\alpha}} \left(
        \mu(\Acal')\int_{\Acal'}f'(a')\frac{1}{\mu(\Acal')}d\mu(a')\right)^{\frac{\alpha}{m+\alpha}}\\
        & = C (L')^{\frac{m}{m+\alpha}} (A_{max})^{-\frac{\alpha m}{m+\alpha}} (A_{max})^{\frac{m\alpha}{m+\alpha}}\left[\EE_{a' \sim \unif_\Acal'} f'(a')\right]^\frac{\alpha}{m+\alpha}\\
        & = C (L')^{\frac{m}{m+\alpha}} \left[\EE_{a' \sim \unif_\Acal'} f'(a')\right]^\frac{\alpha}{m+\alpha}
    \end{align*}
    Thus, as long as $\Acal$ is bounded, the result in \pref{thm:is-alpha-smooth} holds in all generality.
\end{remark}

\subsection{\pfref{lem:policy-smoothing}}\label{sec:policy-smoothing-lem-proof}

Consider any policy $\pi(a\mid s)$. The core of this proof relies on the construction of a ``smoothed", intermediary policy $\pi_K(a\mid s)$. We define $\pi_K$ as follows:

\begin{definition}\label{def:pi-K-def}[Intermediary policy]
    Let $\pi(\cdot \mid s)\in \Delta(\Acal)$ be a given policy that we wish to evaluate. Then, for all $s\in \Scal$ and $K\in \RR^+$, define the intermediary policy $\pi_K(\cdot\mid s)\in \Delta(\Acal)$ as the distribution that samples $a\in \Acal$ by first sampling $a'\sim \pi(\cdot\mid s)$, and then $a \sim \unif_{B_\infty(a',K^{-1/m}/2) \cap \Acal}$, where $B_\infty(a',K^{-1/m}/2)=\{a\in \RR^{m} : \|a-a'\|_\infty\leq K^{-1/m}/2\}$ is the $\ell_\infty$ ball of radius $K^{-1/m}/2$ in $\RR^m$. 
\end{definition}

In other words, the intermediary policy $\pi_K$ is given by sampling in an $m$-dimensional hypercube of side length $K^{-1/m}$ centered at action $a'$ sampled from the original policy $\pi$. $\pi_K$ is trivially continuous in the action $a$. Furthermore, since the volume of the hypercube is $1/K$, $\pi_K$ has a density bounded by $K$. We now introduce a result that describes the behavior of $\EE_{a\sim \pi_K(\cdot \mid s)}[f(a)]$, where $f(a)$ is an $\alpha$-H\"older continuous function. 

\begin{lemma}\label{lem:pi-K-exp}
    Let $C^{\alpha, L}$ be the class of $\alpha$-H\"older continuous functions $f:\Acal\rightarrow [0,1]$ with $\alpha\in(0, 1]$ and norm at most $L$ as in \pref{def:a-smooth}. Furthermore, let $\pi(\cdot \mid s)\in \Delta(\Acal)$ be a given policy and let $\pi_K(\cdot \mid s)$ be its intermediary policy as in \pref{def:pi-K-def}. Then, the following inequality holds:
    \begin{align*}
        \sup_{f\in C^{\alpha, L}} \left(\EE_{a\sim \pi_K(\cdot\mid s)}[f(a)] - \EE_{a\sim \pi(\cdot\mid s)}[f(a)]\right) \leq \sqrt{m}LK^{-\frac{\alpha}{m}}.
    \end{align*}
\end{lemma}

\begin{proof}[\pfref{lem:pi-K-exp}]
   For $\alpha$-H\"older continuous functions with H\"older norm bounded by $L$ we have that $|f(a)-f(a')|\leq L \|a-a'\|_2^\alpha$ for any $a, a'\in \Acal$. This allows us to bound the difference in expectations under a given policy $\pi(\cdot \mid a)$ and the corresponding intermediary policy $\pi_K(\cdot\mid a)$:
    \begin{align*}
        &\sup_{f\in C^{\alpha, L}} \left(\EE_{a\sim \pi_K(\cdot\mid s)}[f(a)] - \EE_{a\sim \pi(\cdot\mid s)}[f(a)]\right) \\
        & \quad = \sup_{f\in C^{\alpha, L}} \left(\EE_{a'\sim \pi(\cdot\mid s)}\left[\EE_{a\sim \unif_{B_{\Acal}(a', K^{-1/m})\cap \Acal}}f(a)\right] - \EE_{a\sim \pi(\cdot\mid s)}[f(a)]\right)\\
        & \quad \leq \sup_{f\in C^{\alpha, L}} \left(\EE_{a'\sim \pi(\cdot\mid s)}\left[\EE_{a\sim \unif_{B_{\Acal}(a', K^{-1/m})\cap \Acal}}[f(a')+\sqrt{m}LK^{-\alpha/m}]\right] - \EE_{a\sim \pi(\cdot\mid s)}[f(a)]\right)=\sqrt{m}LK^{-\alpha/m}
    \end{align*}
    where the inequality comes from $|f(a)-f(a')|\leq L \|a-a'\|_2^\alpha\leq L(\sqrt{m}\|a-a'\|_\infty)^\alpha \leq \sqrt{m} LK^{-\alpha/m}$ for $a, a'\in B_\infty(a', K^{-1/m}/2)\cap \Acal$. 
\end{proof}

\begin{remark}
    For $\alpha>1$, we cannot easily leverage the additional smoothness properties of the class $C^{\alpha, L}$. To understand this, let $m=1$, $r\ll 1$, and consider the Taylor expansion of $f(a')$ around $f(a)$ such that $\|a-a'\|_2=r$. The Taylor expansion contains a term that is bounded by $Lr$ which is of leading order given $r\ll 1$. In other words, for small $r$, functions of smoothness $\alpha\geq1$ are locally linear with derivatives bounded by $L$. Thus, for $\alpha>1$, we cannot obtain a bound better than $Lr$ using similar arguments as above. Note that $Lr$ is still a valid bound in this regime since the functions are still Lipschitz with constant $L$. 
\end{remark}

Using the result of \pref{lem:pi-K-exp}, we can obtain the conclusion of \pref{lem:policy-smoothing} by applying the performance difference lemma, and bounding the difference between $V(\pi; R, \tilde \Mcal)$ and $V(\pi_K; R, \tilde \Mcal)$ for any model $\tilde \Mcal$ allowed by $(\Phi,\Psi)$, where $\pi_K$ is the intermediary policy of $\pi$ (\pref{def:pi-K-def}). We give details below:
\begin{align*}
    |V(\pi; R, \tilde \Mcal) - V(\pi_K; R, \tilde \Mcal)| &\leq\sum_{h=0}^{H-1} \EE_{s_h \sim \PP_h^{\pi_K}} \Big| \EE_{a_h \sim \pi(s_h)}[Q^{\pi, \tilde\Mcal}(s_h, a_h)] - \EE_{a_h \sim \pi_K(s_h)}[Q^{\pi, \tilde\Mcal}(s_h, a_h)] \Big| \\
    &\leq \sum_{h=0}^{H-1} \EE_{s_h \sim \PP_h^{\pi_K}} \Big| \EE_{a_h \sim \pi(s_h)}[R_h(s_h, a_h)] - \EE_{a_h \sim \pi_K(s_h)}[R_h(s_h, a_h)] \\
    &\qquad\qquad\qquad\quad +  \EE_{a_h \sim \pi(s_h)}  \EE\left[ V_{h+1}^{\pi, \tilde\Mcal}(s_{h+1})\mid s_h, a_h, \tilde\Mcal \right] \\
    &\qquad\qquad\qquad\quad -  \EE_{a_h \sim \pi_K(s_h)}  \EE\left[ V_{h+1}^{\pi, \tilde\Mcal}(s_{h+1})\mid s_h, a_h, \tilde\Mcal \right] \Big| \,.
\end{align*}

Now, by assumption $R_h$ is $\alpha$-H\"older continuous in $a$ with H\"older norm at most $L$. In addition, $R$ is a sparse reward so $V^{\pi, \tilde\Mcal}_{h+1} : \Scal \mapsto [0,1]$, which along with the assumption that $\phi_h^*(s,a)\psi^*_h(\cdot)$ is also $\alpha$-H\"older continuous under the total variation distance, implies that $\EE\left[ V_{h+1}^{\pi, \tilde\Mcal}(s_{h+1})\mid s_h, a_h, \tilde\Mcal \right]$ is also $\alpha$-H\"older continuous in $a$ with constant $L$. Therefore, we can apply the result of \pref{lem:pi-K-exp} to control the terms above and obtain the desired result:
\begin{align*}
    |V(\pi; R, \tilde \Mcal) - V(\pi_K; R, \tilde \Mcal)| &\leq \sum_{h=0}^{H-1} \EE_{s_h \sim \PP_h^{\pi_K}} \left(\sqrt{m}LK^{-\alpha/m}+ \sqrt{m}LK^{-\alpha/m}\right) \\
    &\leq 2\sqrt{m}H L K^{-\alpha/m} \,.
\end{align*}

\begin{remark}
    While the assumption of the theorem is that both $R$ and $\phi_h^*(s,a)\psi^*_h(\cdot)$ are  $\alpha$-H\"older continuous in $a$ with H\"older norm at most $L$, this is without loss of generality. To see this, assume instead that $R$ is $\alpha_R$-smooth with $\alpha_R\in(0,\infty)$ and H\"older norm $L_R$, and similarly $\phi_h^*(s,a)\psi^*_h(\cdot)$ is $\alpha_T$-smooth with $\alpha_T\in(0,\infty)$ and H\"older norm $L_T$. Then, we could simply set $\alpha:=\min(\alpha_R, \alpha_T, 1)$ and $L:=\max(L_R, L_T)$ and the result of \pref{lem:policy-smoothing} applies due to the nesting of $\alpha$-smooth classes.  
\end{remark}

\subsection{\pfref{lem:smooth-transitions}}

\begin{proof}
    For any given $h,\phi,\psi,s,a,a'$ we have
    \begin{align*}
        &|\textup{err}^{\phi, \psi}_h(s,a') - \textup{err}^{\phi, \psi}_h(s,a)| \\
        &\leq \Big\| \Big( \phi(s,a')^T\psi(\cdot) - \phi_h^*(s,a')^T\psi_h^*(\cdot) \Big) \\
        &\qquad - \Big( \phi(s,a)^T\psi(\cdot) - \phi_h^*(s,a)^T\psi_h^*(\cdot) \Big) \Big\|_{\textup{TV}} \\
        &\leq 2 \sup_{\phi \in \Phi, \psi \in \Psi} \Big\| \phi(s,a')^T\psi(\cdot) - \phi_h(s,a)^T\psi_h(\cdot) \Big\|_{\textup{TV}} \\
        &\leq 2 U \sup_{\phi \in \Phi} \| \phi(s,a') - \phi(s,a) \|_2 \\
        &\leq 2 U L_\Phi \|a' - a\|^\alpha \,,
    \end{align*}
    where the first inequality follows from the triangle inequality, the second follows from \pref{assum:realizability}, the third follows from H\"older's inequality and the assumd bound on $\Psi$, and the final inequality follows from H\"older continuity of all $\phi \in \Phi$. The conclusion immediately follows.
\end{proof}

\subsection{\pfref{thm:continous-pi-guarantee}}

\begin{proof}[\pfref{thm:continous-pi-guarantee}]

Since $m \geq 1$, we cannot simply use \pref{thm:is-alpha-smooth} to replace the simple importance sampling bound applied by \citet{agarwal2020flambe}, and conclude by simply replacing $|\Acal|$ in the result of \citet{agarwal2020flambe} with some complex term involving the constants $c$, $L_E$, $\alpha_E$ as in \pref{assum:alpha-smooth-errors}. Instead, we must carefully percolate the weaker bound given by \pref{thm:is-alpha-smooth} through the arguments of \citet[Appendix B]{agarwal2020flambe}. In this section, we briefly summarize this process.  

Let us refer to \citet[Lemma $x$]{agarwal2020flambe} by AKKS-L$x$, for each $x$, and similarly AKKS-C$x$ for corollaries. We will briefly describe how updated versions of of the lemmas and corollaries in \citet[Appendix~B]{agarwal2020flambe} can be obtained under our assumptions, by using \pref{thm:is-alpha-smooth} instead of the simple importance sampling bound. We sketch these arguments and results, with formal details omitted when the implications are straightforward. In these arguments we will assume without further comment that $\lambda \leq 1/d$. 

\paragraph{AKKS-L17}
First, we note that at termination we have
\begin{align*}
    &\sup_\pi \EE \left[ \phi_{H-1}(s_{H-1}, a_{H-1})^\top \Sigma_T^{-1} \phi_{H-1}(s_{H-1}, a_{H-1}) \mid \widetilde{\Mcal}, \pi \right]  \\
    & \leq  \sup_\pi \EE \left[ \phi_{H-1}(s_{H-1}, a_{H-1})^\top \Sigma_{T-1}^{-1} \phi_{H-1}(s_{H-1}, a_{H-1}) \mid \widetilde{\Mcal}, \pi \right] \\
    & \leq  \EE \left[ \phi_{H-1}(s_{H-1}, a_{H-1})^\top \Sigma_{T-1}^{-1} \phi_{H-1}(s_{H-1}, a_{H-1}) \mid \widetilde{\Mcal}, \pi_T \right] + \frac{\beta}{2} \\
    & \leq  \beta \,,
\end{align*}
where the first inequality follows because $\Sigma_T - \Sigma_{T-1}$ is positive semi-definite (PSD), the second equality follows because inequality follows due to the approximate optimization condition, and the third follows from the termination condition. Therefore, since $\Sigma_\rho  = \frac{1}{T} \sum_{t=1}^T \Sigma_{\pi_t}$ and $\Sigma_T = \sum_{t=1}^T \Sigma_{\pi_t} + I$, we have 
\begin{equation*}
    \forall \pi : \EE \left[ \phi_{H-1}(s_{H-1}, a_{H-1})^\top (\Sigma_\rho + I/T)^{-1} \phi_{H-1}(s_{H-1}, a_{H-1}) \mid \widetilde{\Mcal}, \pi \right]  \leq T \beta \,.
\end{equation*}

Finally, we note that since at every non-terminating iteration the objective is at least $\beta/2$, we can apply an identical potential-based argument as in the original proof of AKKS-L17, except with $\beta$ replaced with $\beta/2$, which gives us the bound
\begin{equation*}
    T \leq 8d \log(1 + 8/\beta) / \beta \,.
\end{equation*}
That is, by using our alternative elliptical planner that only approximately optimizes the objective each iteration, we obtain the same guarantee as in AKKS-L17, except the bound on $T$ is a constant-factor weaker.

\paragraph{AKKS-C5}
The statement and proof of this corollary is identical since no aspect of it relies on $\Acal$ being discrete.

\paragraph{AKKS-L6}
For this lemma, we can obtain a similar result, by relying on both \pref{assum:alpha-smooth-errors} and \pref{thm:is-alpha-smooth} instead of the simple importance sampling bound. We simplify notation by leting $U(\alpha_E, L_E):=cL_E^{\frac{m}{m+\alpha_E}}$ where $c$ is the constant in \pref{thm:is-alpha-smooth}. Thus, we can obtain the bound 
\begin{align*}
    &\EE_{s_h \sim P_h^\pi, a_h \sim \pi(s_h)}[\one\{s_h \neq s_{\textup{absorb}}\} \textup{err}(s_h,a_h) ] \\
    &\leq \EE_{s_h \sim P_h^\pi}\left[\one\{s_h \neq s_{\textup{absorb}}\} U(\alpha_E, L_E) \EE_{a_h \sim \unif_\Acal}[\textup{err}(s_h,a_h)]^{\alpha_E/(m+\alpha)} \right] \\
    &\leq U(\alpha_E, L_E) \EE_{s_h \sim P_h^\pi, a_h \sim \unif_\Acal}[\one\{s_h \neq s_{\textup{absorb}}\} \textup{err}(s_h,a_h) ]^{\alpha_E/(m+\alpha_E)} \\
    &\leq U(\alpha_E, L_E) \epstv^{\alpha_E/(2(m+\alpha_E))} \,,
\end{align*}
where the first inequality follows from \pref{assum:alpha-smooth-errors}, the second follows from Jensen's inequality, and the last following the same logic as in AKKS-L6. Therefore, following the remaining steps as in AKKS-L6, we obtain the same result except with $U(\alpha_E, L_E)\epstv^{\alpha_E/(2(m+\alpha_E))}$ in place of $K \sqrt{\epstv}$. 

\paragraph{AKKS-L7 and AKKS-C8 }
Plugging in the previous result from AKKS-L6, we obtain an almost-identical results, except with $K \sqrt{\epstv}$ replaced everywhere by $U(\alpha_E, L_E)\epstv^{\alpha_E/(2(m+\alpha_E))}$. In particular, this gives us a bound on the maximum number of rounds 
\begin{equation*}
    \jmax = \frac{4Hd}{\lambda U(\alpha_E, L_E) \epstv^{\alpha_E/(2(m+\alpha_E))}}  \log\left(1 + \frac{4H}{\lambda U(\alpha_E, L_E) \epstv^{\alpha_E/(2(m+\alpha_E))}} \right) \,,
\end{equation*}
in order to obtain a round where
\begin{equation*}
    \PP[(s_h,a_h) \notin \Kcal(\Sigma_{j,j}) \mid \pi, \widehat\Mcal^{(j)}] \leq 2 H U(\alpha_E, L_E) \epstv^{\alpha_E/(2(m+\alpha_E))}
\end{equation*}

\paragraph{AKKS-L9}
We can obtain a very similar lemma here, except that we need to interpret $K$ differently. Here, we will interpret $K$ as some constant such that the density ratio $\pi(a \mid s) / \unif_\Acal(a)$ is bounded by $K$ for every $s,a$. Them, the same result follows using an identical argument involving importance sampling, except that the value $K$ in the bound is $\pi$-dependent. We note that this implicitly requires $\pi$ to be non-deterministic.

\paragraph{AKKS-L10}
We will obtain a similar lemma here, except that the lemma will only apply to policies $\pi$ that have a bounded density ratio $\sup_{s,a} \pi(a \mid s) / \unif_\Acal(a) \leq K$ for some $K$, and the bounds will be $K$-dependent.

First, applying successively our updated versions of AKKS-L6, AKKS-L9, and AKKS-C8, we have
\begin{align*}
    &\PP[(s_h,a_h) \notin \Kcal(\Sigma_{j,j}) \mid \pi, \widehat\Mcal^{(j)}] \\
    &\leq \alpha H^2 K U(\alpha_E, L_E)\epstv^{\alpha_E/(2(m+\alpha_E))} + \frac{T \beta}{2 \alpha} + \frac{\alpha d}{2 T} + H U(\alpha_E, L_E)\epstv^{\alpha_E/(2(m+\alpha_E))} \overset{\Delta}{=} \epsescape \,.
\end{align*}

Then, tracing the rest of the proof with the updated version of AKKS-L6, we get a final result of
\begin{equation*}
    \max_{\pi, R} \left| V(\pi;R,\widehat{\Mcal}) - V(\pi;R,\Mcal) \right| \leq H U(\alpha_E, L_E)\epstv^{\alpha_E/(2(m+\alpha_E))} + H \epsescape \,.
\end{equation*}

\paragraph{Final steps.} We put together the updated lemmas and corollaries to obtain the final $K$-dependent bound on the sample complexity in order to accurately evaluate all policies $\pi$ that satisfy $\sup_{s,a} \pi(a \mid s) / \unif_\Acal(a) \leq K$. To simplify our notation, we introduce the following abbreviations: $\tau=m/\alpha_E$ and $\kappa=m/(m+\alpha_E)$. 

First, we note that from AKKS-10, the policy evaluation error is given by
\begin{align*}
    \textup{error} &= \alpha H^3 K U(\alpha_E, L_E)\epstv^{\alpha_E/(2(m+\alpha_E))} + \frac{H T \beta}{2 \alpha} + \frac{H \alpha d}{2 T} + (H^2 + H) U(\alpha_E, L_E)\epstv^{\alpha_E/(2(m+\alpha_E))}  \\
    & = \alpha H^3 K U(\alpha_E, L_E)\epstv^{1/(2(1+\tau))} + \frac{H T \beta}{2 \alpha} + \frac{H \alpha d}{2 T} + (H^2 + H) U(\alpha_E, L_E)\epstv^{1/(2(1+\tau))}  \\
    &\leq \alpha H^3 K U(\alpha_E, L_E)\epstv^{1/(2(1+\tau))} + \frac{H T \beta}{2 \alpha} + \frac{H \alpha d}{2 T} + 2H^2 U(\alpha_E, L_E)\epstv^{1/(2(1+\tau))}  \,,
\end{align*}
Next, optimizing over $\alpha$ in this bound and setting
\begin{equation*}
    \alpha = \sqrt{\frac{HT\beta}{2} \left( H^3 K U(\alpha_E, L_E)\epstv^{1/(2(1+\tau))} + \frac{H d}{2 T} \right)^{-1}}
\end{equation*}
gives us
\begin{equation*}
    \textup{error} \leq \sqrt{\left( 2 H^4 K T\beta U(\alpha_E, L_E)\epstv^{1/(2(1+\tau))} + H^2 d \beta \right)} + 2H^2 U(\alpha_E, L_E)\epstv^{1/(2(1+\tau))} \,.
\end{equation*}

Next, we note that $\beta$ is a free parameter of the algorithm, which we may set arbitrarily low at the cost of greater sample complexity. We balance the first term by setting $\beta$ according to
\begin{equation*}
    \beta = \frac{2 H^2 K U(\alpha_E, L_E)\epstv^{1/(2(1+\tau))}}{d} \,,
\end{equation*}
which gives us
\begin{equation*}
    \textup{error} \leq \sqrt{2} H^2 \sqrt{T \beta + 1} \sqrt{K U(\alpha_E, L_E) } \epstv^{1/(4(1+\tau))} + 2H^2 U(\alpha_E, L_E) \epstv^{1/(2(1+\tau))} \,.
\end{equation*}
We note that $U(\alpha_E, L_E)+\sqrt{U(\alpha_E, L_E)}\leq 2\max(U(\alpha_E, L_E),1)$. For now, assume that $U(\alpha_E, L_E)\geq 1$. We further apply the fact that $K \geq 1$, $\epstv \leq 1$, and $T\beta \leq 8d \log(1 + 8/\beta)$ in order to get
\begin{equation*}
    \textup{error} \leq 8 \sqrt{d} \sqrt{\log(1+8/\beta)} H^2 \sqrt{K} U(\alpha_E, L_E) \epstv^{1/(4(1+\tau))} \,.
\end{equation*}
Now, we would like to set $\epstv$ sufficiently small such that the error is less than $\epsilon$. We can achieve this if we set
\begin{equation*}
    \epstv = 8^{-4-4\tau} \frac{H^{-8-8\tau} U(\alpha_E,L_E)^{-4-4\tau} K^{-2-2\tau} d^{-2-2\tau} \epsilon^{4+4\tau}}{\log(1+8/\beta')^{2+2\tau}} \,,
\end{equation*}
as long as $\beta'$ satisfies
\begin{align*}
    &\qquad \beta' \leq \beta = \frac{2 H^{2} K U(\alpha_E, L_E) \epstv^{1/(2(1+\tau))}}{d} = \frac{\epsilon^2}{2^5 d^2 H^2 U(\lambda_T, L_E) \log(1 + 8/\beta')} \\
    &\iff 8/\beta' \geq  2^8 d^2 H^2 U(\alpha_E, L_E) \epsilon^{-2} \log(1+8/\beta') \,.
\end{align*}
We note that $\log(1 + x) \leq x/\sqrt{x+1}$ for $x\geq0$, so this condition on $\beta'$ is clearly satisfied if we set
\begin{equation*}
    \beta' = \frac{8}{2^{16} d^4 H^4 U^2(\alpha_E, L_E) \epsilon^{-4} - 1}\,.
\end{equation*}

Next, we need to reason about how small $\lambda$ and $n$ must be set to achieve this. Recall that
\begin{equation*}
    \epstv \leq \lambda d + \frac{2 \log(\jmax H |\Phi| |\Psi| / \delta)}{n} \,.
\end{equation*}
This implies that we can achieve the required bound as long as we set
\begin{align*}
    \lambda &= \frac{8^{-4-4\tau}}{2} \frac{H^{-8-8\tau} U(\alpha_E, L_E)^{-4-4\tau} K^{-2-2\tau} d^{-3-2\tau} \epsilon^{4+4\tau} }{\log(1+8/\beta')^{2+2\tau}}\\
    n &= 8^{4+4\tau} \frac{H^{8+8\tau} U(\alpha_E, L_E)^{4+4\tau} K^{2+2\tau} d^{2+2\tau} \log(1+8/\beta')^{2+2\tau}}{\epsilon^{4+4\tau} } \log(\jmax H |\Phi| |\Psi| / \delta)
\end{align*}

Furthermore, recalling the definition of $\jmax$ and disregarding the log terms, we have
\begin{align*}
    \jmax &= \frac{4Hd}{\lambda U(\alpha_E, L_E) \epstv^{1/(2(1+\tau))}}  \log\left(1 + \frac{4H}{\lambda U(\alpha_E, L_E) \epstv^{1/(2(1+\tau))}} \right) \\
    &= \tilde\Ocal \left(\frac{Hd}{\lambda U(\alpha_E, L_E) \epstv^{1/(2(1+\tau))}} \right) \\
    &= \tilde\Ocal \left( \frac{H^{13+8\tau} U(\alpha_E, L_E)^{5+4\tau} K^{3+2\tau} d^{5+2\tau}} {\epsilon^{6+4\tau} }\right)\,.
\end{align*}

Finally, the total sample complexity is given by $n \jmax H$. Plugging in the above values for $n$, $\lambda$, and $\epstv$, and disregarding log terms, we have
\begin{align}
    n \jmax H &= \tilde\Ocal \left( H^{22+16\tau} K^{5+4\tau} d^{7+4\tau} \left(\frac{1}{\epsilon}\right)^{10+8\tau} U(\alpha_E, L_E)^{9+8\tau} \ln\left(\frac{|\Phi||\Psi|}{\delta}\right)\right)  \,.
\end{align}
Furthermore, recalling that $U(\alpha_E, L_E)=cL_E^{\frac{m}{m+\alpha_E}}=cL_E^\kappa$, we obtain the desired sample complexity:
\begin{align}\label{eqn:final-result}
      n \jmax H &=\tilde\Ocal \left( H^{22+16\tau} K^{5+4\tau} d^{7+4\tau} \left(\frac{1}{\epsilon}\right)^{10+8\tau} L_E^{(9+8\tau)\kappa} \ln\left(\frac{|\Phi||\Psi|}{\delta}\right)\right)  \,.
\end{align}
\end{proof}

\subsection{\pfref{thm:continous-r-guarantee}}

\begin{proof}[\pfref{thm:continous-r-guarantee}]
    The proof of this theorem is an immediate application of \pref{thm:continous-pi-guarantee} for the smoothed, intermediary policy $\pi_K$ as in \pref{def:pi-K-def}. First, note that we can decompose the error of the policy $\pi$ in terms of errors of the policy $\pi_K$:
\begin{align*}
    |V(\pi; R, \Mcal) - V(\pi; R, \widehat\Mcal)| &\leq |V(\pi; R, \Mcal) - V(\pi_K; R, \Mcal)| \\
    &\quad + |V(\pi_K; R, \Mcal) - V(\pi_K; R, \widehat\Mcal)| \\
    &\quad + |V(\pi_K; R, \widehat\Mcal) - V(\pi; R, \widehat\Mcal)| 
\end{align*}
We can bound the first and third term on the RHS of the above inequality using the result of \pref{lem:policy-smoothing}. To see this, let $\sigma=m/\min(\alpha_T, \alpha_R)$ and $L=\max(L_T, L_R)$. Then, by setting $K=\left(\frac{8\sqrt{m}HL}{\epsilon}\right)^{\sigma}$, \pref{lem:policy-smoothing} yields:
\begin{align*}
    |V(\pi; R, \Mcal) - V(\pi; R, \widehat\Mcal)| & \leq |V(\pi_K; R, \Mcal) - V(\pi_K; R, \widehat\Mcal)| +\frac{\epsilon}{2}
\end{align*}
Since $\pi_K$ has bounded density, we can control the remaining term on the RHS by leveraging the result of \pref{thm:continous-pi-guarantee} with $K=\left(\frac{8\sqrt{m}HL}{\epsilon}\right)^{\sigma}$ and $\epsilon$ replaced with $\epsilon/2$. This would ensure that  $|V(\pi_K; R, \Mcal) - V(\pi_K; R, \widehat\Mcal)| \leq \epsilon / 2$ with probability at least $1-\delta$. Therefore, putting this together and plugging in this value of $K$ into our result from \pref{thm:continous-pi-guarantee}, we get our desired performance result of
\begin{equation*}
    |V(\pi; R, \Mcal) - V(\pi; R, \widehat\Mcal)| \leq \epsilon \,,
\end{equation*}
with a total number of trajectories collected given by 
\begin{align*}
    \tilde O &\Bigg( H^{22+16\tau+(4 \tau + 5)\sigma} d^{7+4\tau} \Big(\frac{1}{\epsilon}\Big)^{10+8\tau+(4 \tau + 5)\sigma} L_E^{(9 + 8 \tau)\kappa} L^{(4 \tau + 5)\sigma} \ln\Big(\frac{|\Phi||\Psi|}{\delta}\Big) \Bigg)\,,
\end{align*}
where $\sigma = m / \min(\alpha_T, \alpha_R)$, and $L = \max(L_T, L_R)$.
Furthermore, the corresponding values of the hyperparameters $\beta$, $n$, and $\jmax$ that achieve these results are the same as those from the proof of \pref{thm:continous-pi-guarantee}, plugging in the above value of $K$.
\end{proof}

\section{CASE STUDY: RAFFLE}
\label{app:raffle}

In this section, we apply our framework to RAFFLE \citep{cheng2023improved}, a recently proposed reward-free low-rank MDP algorithm which was shown to achieve near optimal sample complexity $\tilde{O}\left(\frac{H^3d^2K(d^2+K)}{\epsilon^2}\right)$ (Theorem 2 in \citet{cheng2023improved}). The significant improvement in sample complexity over FLAMBE is attributed to a novel exploration-driven pseudo-reward mechanism which targets sample collection towards high-error regions. 

We now seek to derive a similar sample complexity for the RAFFLE algorithm when the action space is continuous, specifically $\Acal = [0,1]^m$. We first note that the size $K$ of the discrete action space in RAFFLE enters the bound in their Theorem 2 via the application of importance sampling (IS) bounds, as in our \pref{lem:is}. The IS bound is applied with two different choices of $f(s,a)$: (1) the (squared) TV error of transition distributions (as in our \pref{assum:alpha-smooth-errors}), denoted by $f_h^{(n)}(s,a)$; and (2) their (squared) bonus reward function terms, given by $\hat{b}_h^{(n)}(s,a)$. Thus, in order to provide a similar bound, it suffices to make $f_h^{(n)}(s,a)$ and $\hat{b}_h^{(n)}(s,a)$ $\alpha$-smooth in $a$ and replace the IS bounds with the bound in our \pref{thm:is-alpha-smooth}. 

The TV error smoothness assumption is already outlined in \pref{assum:alpha-smooth-errors} which we take to hold for RAFFLE as well. Next, we see that the bonus reward function is the matrix norm of $\phi^{(n)}_h(s, a)$ induced by a symmetric matrix with eigenvalues smaller that $\lambda_n^{-1}$. Thus, we introduce the following assumption: 

\begin{assum}[$\alpha$-smooth feature vectors] \label{assum:alpha-smooth-features} There exists $\alpha_\Phi\in (0, \infty)$ such that the scalar components $\phi_i(s,a)$, $i \in [d]$ are $\alpha_\Phi$-smooth in $a\in \Acal$ with norm $L_\Phi$ for every $\phi \in \Phi$ and $s\in \Scal$. 
\end{assum}
\pref{assum:alpha-smooth-features} is similar to the assumption that enables smooth TV errors in \pref{lem:smooth-transitions}. However, \pref{assum:alpha-smooth-features} is more general in that it allows for higher degrees of smoothness when $\alpha_\Phi\geq 1$.

\begin{remark}
    RAFFLE applied IS bounds to $f(s,a)^2=\left\|\phi(s,a)^\top\psi(\cdot) - \phi_h^*(s,a)^\top \psi_h^*(\cdot)\right\|_{\textup{TV}}^2$ and $g(s,a)^2=\left\|\hat{\phi}_h^{(n)}(s,a)\right\|_U^2$, involving sums over $U_{ji} \hat{\phi}_{h,i}^{(n)}(s,a)\hat{\phi}_{h,j}^{(n)}(s,a)$. While \pref{assum:alpha-smooth-errors} and \pref{assum:alpha-smooth-features} only impose smoothness on $f(s,a)$ and each of the $\hat{\phi}_{h,i}^{(n)}(s,a)$'s, respectively, the space of $\alpha$-smooth functions as defined in \pref{def:a-smooth} is closed under multiplication \citep[Theorem 4.2 and Proposition 4.3]{behzadan2021multiplication}, making $f(s,a)^2$ $\alpha_E$-smooth and $\hat{\phi}_{h,i}^{(n)}(s,a)\hat{\phi}_{h,j}^{(n)}(s,a)$ $\alpha_\Phi$-smooth. In the remainder of this section, we extend our notation, denoting the $\alpha$-smooth norms of $f^2(s,a)$ and the $\hat{\phi}_{h,i}^{(n)}(s,a)\hat{\phi}_{h,j}^{(n)}(s,a)$ products by $L_E$ and $L_\Phi$, respectively.
\end{remark}

With \pref{assum:alpha-smooth-errors} and \pref{assum:alpha-smooth-features}, we can replace the IS bounds with the bound in our \pref{thm:is-alpha-smooth} and carefully propagate these arguments through sample complexity proof in \citet[Appendix A]{cheng2023improved}. In the remainder of this section, we sketch this process and results, with formal details omitted for brevity. 
We provide the sample complexity of the RAFFLE algorithm with continuous action spaces in the following theorem:

\begin{theorem}[RAFFLE guarantees for continuous $\Acal$]\label{thm:raffle-guarantee}
    Assume $\Mcal$ has rank $d$ and \pref{assum:realizability}, \pref{assum:alpha-smooth-errors}, and \pref{assum:alpha-smooth-features} hold. Let $\alpha=\min(\alpha_E, \alpha_\Phi)$, $L=\max(L_E, L_\Phi)$, and $\tau=m/(m+\alpha)$. Then, if $\tau<1/2$ and we set $\hat{\alpha}_n = \tilde{O}\left(\sqrt{n^{\tau}\left(L^{\tau} +(dL)^{\tau} + d^2\right)}\right)$, $\lambda_n=\tilde{O}(d)$, with probability at least $1-\delta$, RAFFLE returns a model $\widehat{\Mcal}$ that satisfies 
    \begin{equation*}
    \left| V(\pi;R,\widehat\Mcal) - V(\pi;R,\Mcal) \right| \leq \epsilon
    \end{equation*}
for every sparse reward function $R$, and every policy $\pi$. The total number of trajectories collected by RAFFLE is upper bounded by
\begin{equation*}
    \tilde{O}\left(H^{1+\textstyle\frac{2}{1-2\tau}}d^{\textstyle\frac{2}{1-2\tau}}\left(\frac{1}{\epsilon}\right)^{\textstyle\frac{2}{1-2\tau}}L^{\textstyle\frac{\tau}{1-2\tau}}(L^\tau + (Ld)^\tau+d^2)^{\textstyle\frac{1}{1-2\tau}}\right).
\end{equation*}
\end{theorem}
We note that as $\tau\to 0$ (indicating that $\varepsilon$ and $\phi(s,a)$ have a very high order of smoothness relative to $m$) we achieve identical dependence on $H$, $d$, and $1/\epsilon$ as in \citet{cheng2023improved}. 

\begin{remark}
    Although this theorem is not applicable to $\varepsilon$ and $\phi(s,a)$ with a lower degree of smoothness ($\tau\geq 1/2$), it is worth noting that a more refined bound could potentially be achieved through an alternative proof strategy, instead of naively replacing the IS arguments in \citet{cheng2023improved}'s proof with our \pref{thm:is-alpha-smooth}. We leave this for future work. 
\end{remark}
 
\begin{proof}[\pfref{thm:raffle-guarantee}] Similarly to the FLAMBE analysis, we refer to \citet[Lemma $x$, Appendix $A$]{cheng2023improved} by CHYL-L$x$, for each $x$, and similarly CHYL-C$x$  and CHYL-P$x$ for corollaries and propositions, respectively. Whenever possible, we keep notation consistent with that in \citet{cheng2023improved}.

\paragraph{CHYL-L1} 
The statement and proof of this lemma remains unchanged, as its validity does not depend on $\Acal$ being discrete. 

\paragraph{CHYL-C1} Let $\alpha=\min(\alpha_E, \alpha_\Phi)$ and $L=\max(L_E, L_\Phi)$. Setting 
\[\hat{\alpha}_n=5\sqrt{\beta_3 n \zeta_n^{\frac{\alpha}{m+\alpha}}\left( L^{\frac{m}{m+\alpha}} + (d L)^{\frac{m}{m+\alpha}} + d^2\right)}\] 
in the definition of the bonus reward term $\hat{b}_h(s,a)$, the statement and result of this lemma stay the same. 

\paragraph{CHYL-L2} We add an additional condition to the statement of the lemma that $g^2$ is $\alpha$-smooth with norm $L$. 
Then, for all $h\geq 2$ and for all policies $\pi_h$, there exists some $c=O(1)$ such that the following inequality holds:
\begin{align*}
    \mathop{\EE}_{s_h \sim P_{h-1} \atop a_h \sim \pi_h}&[g(s_h,a_h)|s_{h-1},a_{h-1}]\\  &\leq  \left\|\phi_{h-1}(s_{h-1},a_{h-1})\right\|_{(M_{h-1,\phi})^{-1}} \times\nonumber\\
    &\quad 
    \sqrt{cnL^{\frac{m}{m+\alpha}}\mathop{\EE}_{s_{h}\sim(P^\star, \Pi)\atop a_h \sim \Ucal }[g^2(s_h,a_h)]^{\frac{\alpha}{m+\alpha}} + \lambda_n dB^2 + nB^2\mathop{\EE}_{s_{h-1}\sim(P^\star,\Pi)\atop a_{h-1}\sim\Pi }\left[f_{h-1}(s_{h-1},a_{h-1})^2\right]}.
\end{align*}
Here, $\Ucal=\unif_\Acal$. To obtain this result, we replace the importance sampling argument in step (iv) of the proof in \citep{cheng2023improved} with the result in our \pref{thm:is-alpha-smooth}. 

\paragraph{CHYL-L3} We now propagate the result of CHYL-L2 to obtain new bounds for Eq. (11-14) in  \citet[Appendix A]{cheng2023improved}. We first note that for any $h$, we have
\begin{align}
    \label{eq:mean-err}
    &\mathop{\EE}_{s_{h-1} \sim P^\star_{h-1}(\cdot \mid s_{h-2}, a_{h-2}) \atop a_{h-1} \sim \Pi_n(s_{h-1})} \left[ f_{h-1}^{(n)}(s_{h-1}, a_{h-1})^2 \ \Big| \ s_{h-2}, a_{h-2} \right] \\
    &= \int_{s_{h-1}} \int_{a_{h-1}} f^{(n)}(s_{h-1},a_{h-1})^2 \phi^\star_{h-2}(s_{h-2}, a_{h-2})^\top \mu^\star_{h-2}(s_{h-1}) \Pi_n(a_{h-1} \mid s_{h-1}) d a_{h-1} d s_{h-1}  \nonumber\\
    &= \sum_{i=1}^d \Omega_i \phi^\star_{h-2,i}(s_{h-2}, a_{h-2})  \,, \nonumber
\end{align}
where
\begin{equation*}
    \Omega_i = \int_{s_{h-1}} \int_{a_{h-1}} f^{(n)}(s_{h-1},a_{h-1})^2  \mu^\star_{h-2,i}(s_{h-1}) \Pi_n(a_{h-1} \mid s_{h-1}) d a_{h-1} d s_{h-1} \,.
\end{equation*}
Now, by the normalization assumption on $\mu^\star_{h-2}$, and the fact that $f^{(n)}$ is uniformly bounded by 1, it easily follows that $\|\Omega\|_2 \leq \sqrt{d}$, and therefore $\|\Omega\|_1 \leq d$. Therefore, given \pref{assum:alpha-smooth-features}, it easily follows that \pref{eq:mean-err} is uniformly $\alpha_\Phi$-smooth in $a_{h-2}$  with norm at most $d L_\Phi$. Then, applying \pref{thm:is-alpha-smooth} it follows that
\begin{align*}
    &\mathop{\EE}_{\substack{s_{h-1} \sim (P^\star, \Pi_n) \\ a_{h-1} \sim \Ucal}}\left[f^{(n)}_{h-1}(s_{h-1},a_{h-1})^2\right] \\
    &= \mathop{\EE}_{\substack{s_{h-2} \sim (P^\star, \Pi_n) \\ a_{h-2} \sim \Pi_n(s_{h-2})}}\left[ \mathop{\EE}_{s_{h-1} \sim P^\star_{h-1}(\cdot \mid s_{h-2}, a_{h-2}) \atop a_{h-1} \sim \Pi_n(s_{h-1})} \left[ f_{h-1}^{(n)}(s_{h-1}, a_{h-1})^2 \ \Big| \ s_{h-2}, a_{h-2} \right] \right] \\
    &\leq d^{\frac{m}{m+\alpha_\Phi}} L_\Phi^{\frac{m}{m+\alpha_\Phi}} \mathop{\EE}_{\substack{s_{h-2} \sim (P^\star, \Pi_n) \\ a_{h-2} \sim \Ucal}}\left[ \mathop{\EE}_{s_{h-1} \sim P^\star_{h-1}(\cdot \mid s_{h-2}, a_{h-2}) \atop a_{h-1} \sim \Pi_n(s_{h-1})} \left[ f_{h-1}^{(n)}(s_{h-1}, a_{h-1})^2 \ \Big| \ s_{h-2}, a_{h-2} \right] \right]^{\frac{\alpha_\Phi}{m+\alpha_\Phi}} \\
    &= d^{\frac{m}{m+\alpha_\Phi}} L_\Phi^{\frac{m}{m+\alpha_\Phi}} \mathop{\EE}_{\substack{s_{h-2} \sim (P^\star, \Pi_n) \\ a_{h-2}, a_{h-1} \sim \Ucal \\ s_{h-1} \sim P^\star_{h-1}(\cdot \mid s_{h-2}, a_{h-2})}}\left[f^{(n)}_{h-1}(s_{h-1},a_{h-1})^2\right]^{\frac{\alpha_\Phi}{m+\alpha_\Phi}}
\end{align*}

Given this observation, we can follow the steps as in the original proof of CHYL-L3, using the the modified CHYL-L2 above, as well replacing the importance sampling argument at time step $h-2$ with \pref{thm:is-alpha-smooth} using the above observation. Specifically, we have
\begin{align*}
      & \mathop{\EE}_{s_{h}\sim\hat{P}_{h-1}^{(n)}\atop a_{h}\sim \pi }\left[f_{h}^{(n)}(s_{h},a_{h})\bigg|s_{h-1},a_{h-1}\right] \\ 
    & \leq 
    \left\|\hat{\phi}_{h-1}^{(n)}(s_{h-1},a_{h-1})\right\|_{(U_{h-1,\hat{\phi}}^{(n)})^{-1}} \times \\
    & \qquad \sqrt{ cn L_E^{\frac{m}{m+\alpha_E}} \mathop{\EE}_{\substack{s_{h-1} \sim (P^\star, \Pi_n) \\ a_{h-1}, a_h \sim \Ucal \\ s_h \sim P^\star_h}} \Big[ f^{(n)}_h(s_h,a_h)^2 \Big]^{\frac{\alpha_E}{m+\alpha_E}}
    +\lambda_n d 
    + n (d L_\Phi)^{\frac{m}{m+\alpha_\Phi}} \mathop{\EE}_{\substack{s_{h-2} \sim (P^\star, \Pi_n) \\ a_{h-2}, a_{h-1} \sim \Ucal \\ s_{h-1} \sim P^\star_{h-1}}}\left[f^{(n)}_{h-1}(s_{h-1},a_{h-1})^2\right]^{\frac{\alpha_\Phi}{m+\alpha_\Phi}}} \tag{above bound and CHYL-L2} \\
    & \leq \left\|\hat{\phi}_{h-1}^{(n)}(s_{h-1},a_{h-1})\right\|_{(U_{h-1,\hat{\phi}}^{(n)})^{-1}} \sqrt{cn L_E^{\frac{m}{m+\alpha_E}}\zeta_n^{\frac{\alpha_E}{m+\alpha_E}} + cn (d L_\Phi)^{\frac{m}{m+\alpha_\Phi}}\zeta_n^{\frac{\alpha_\Phi}{m+\alpha_\Phi}} + \beta_3n\zeta_nd^2} \tag{CHYL-L1 and $\lambda_n=\beta_3nd\zeta_n$}\\
    & \leq \left\|\hat{\phi}_{h-1}^{(n)}(s_{h-1},a_{h-1})\right\|_{(U_{h-1,\hat{\phi}}^{(n)})^{-1}} \sqrt{\beta_3 n \left( L_E^{\frac{m}{m+\alpha_E}}\zeta_n^{\frac{\alpha_E}{m+\alpha_E}} + (d L_\Phi)^{\frac{m}{m+\alpha_\Phi}}\zeta_n^{\frac{\alpha_\Phi}{m+\alpha_\Phi}} + \zeta_nd^2\right)}\\
    & \leq \alpha_n \left\|\hat{\phi}_{h-1}^{(n)}(s_{h-1},a_{h-1})\right\|_{(U_{h-1,\hat{\phi}}^{(n)})^{-1}}
\end{align*}
where we set $\alpha_n = \hat{\alpha}_n/5 = \sqrt{\beta_3 n \zeta_n^{\frac{\alpha}{m+\alpha}}\left( L^{\frac{m}{m+\alpha}} + (d L)^{\frac{m}{m+\alpha}} + d^2\right)}$. Note that without loss of generality we can let the constant $c$ from \pref{thm:is-alpha-smooth} be absorbed into $\beta_3$ since they are both $O(1)$. 

Eq. (12) follows from similar arguments such that
\begin{align*}
    \mathop{\EE}_{s_{h}\sim P_{h-1}^{*}\atop a_{h\sim \pi}} \left[f_{h}^{(n)}(s_{h},a_{h})\bigg|s_{h-1},a_{h-1}\right]\leq \alpha_n\left\|\phi_{h-1}^{*}(s_{h-1},a_{h-1})\right\|_{(U_{h-1,\phi^\star}^{(n)})^{-1}}.
\end{align*}

For Eq. (13), we have $\hat{b}_h^{(n)}(s,a)=\min\left\{\hat{\alpha}_n\left\|\hat{\phi}_{h}^{(n)}(s,a)\right\|_{(\hat{U}_{h}^{(n)})^{-1}},1\right\}\leq 3\hat{\alpha}_n\left\|\hat{\phi}_{h}^{(n)}(s,a)\right\|_{(U_{h,\hat{\phi}}^{(n)})^{-1}}$ from CHYL-L1. We can then apply CHYL-L2 (before the IS argument):
\begin{align*}
    	& \mathop{\EE}_{s_{h}\sim P^\star_{h-1} \atop a_{h} \sim {\pi_n} }\left[\hat{b}^{(n)}_{h}(s_{h},a_{h})\bigg|s_{h-1},a_{h-1}\right] \leq \left\|\phi_{h-1}^\star(s_{h-1},a_{h-1})\right\|_{(W_{h-1,\phi^\star}^{(n)})^{-1}}
		\sqrt{n\mathop{\EE}_{s_{h}\sim (P^\star, \Pi_n)\atop a_h\sim \pi_n}[\{\hat{b}_h^{(n)}(s_h,a_h)\}^2]+\lambda_n d }.
\end{align*}
We note that $\hat{b}_h^{(n)}(s,a)$ is not $\alpha$-smooth and thus we cannot directly apply our \pref{thm:is-alpha-smooth} results here. Instead, we will have to bound this term by $\|\hat{\phi}_h^{(n)}(s,a)\|_{(U_{h,\hat{\phi}}^{(n)})^{-1}}$ which we now show is $\alpha_\Phi$-smooth with norm at most $L_\Phi$:
\begin{align*}
    \|\hat{\phi}_h^{(n)}(s,a)\|^2_{(U_{h,\hat{\phi}}^{(n)})^{-1}} & = \textup{tr}\left\{\hat{\phi}_h^{(n)}(s_h,a_h)\hat{\phi}_h^{(n)}(s_h,a_h)^T 
    (U_{h,\hat{\phi}}^{(n)})^{-1}\right\}\\
    & = \sum_{i=1}^d\sum_{j=1}^d \{(U_{h,\hat{\phi}}^{(n)})^{-1}\}_{ji}\hat{\phi}_{h,i}^{(n)}(s_h,a_h)\hat{\phi}_{h,j}^{(n)}(s_h,a_h)\\
    & \leq \frac{1}{\lambda_n}\sum_{i=1}^d\sum_{j=1}^d \hat{\phi}_{h,i}^{(n)}(s_h,a_h)\hat{\phi}_{h,j}^{(n)}(s_h,a_h) \tag{from the definition of $(U_{h,\hat{\phi}}^{(n)})^{-1}$}
\end{align*}
By \pref{assum:alpha-smooth-features}, we have that $\|\hat{\phi}_h^{(n)}(s,a)\|^2_{(U_{h,\hat{\phi}}^{(n)})^{-1}}$ is the sum of $d^2$ $\alpha_\Phi$-smooth functions with norm at most $L_\Phi/\lambda_n$. Since the set of $\alpha$-smooth function is closed under addition, it immediately follows that $\|\hat{\phi}_h^{(n)}(s,a)\|^2_{(U_{h,\hat{\phi}}^{(n)})^{-1}}$ is $\alpha_\Phi$-smooth with norm at most $d^2 L_\Phi/\lambda_n=dL_\Phi/(\beta_3\log(2|\Phi||\Psi|nH/\delta)) < dL_\Phi$. Then, we can control the $\hat{b}_h^{(n)}(s,a)$ expression as follows:
\begin{align*}
    n\mathop{\EE}_{s_{h}\sim (P^\star, \Pi_n)\atop a_h\sim \pi_n}[\{\hat{b}_h^{(n)}(s_h,a_h)\}^2]
    &\leq 9n\hat{\alpha}_n^2 \mathop{\EE}_{s_{h}\sim (P^\star, \Pi_n)\atop a_h\sim \pi_n}
    \left[\left\|\hat{\phi}_h^{(n)}(s_h,a_h)\right\|^2_{ (U_{h,\hat{\phi}}^{(n)})^{-1} }\right]\\
    & \leq 9n\hat{\alpha}_n^2 c(dL_\Phi)^{\frac{m}{m+\alpha_\Phi}} 
     \mathop{\EE}_{s_{h}\sim (P^\star, \Pi_n)\atop a_h\sim \Ucal}\left[\left\|\hat{\phi}_h^{(n)}(s_h,a_h)\right\|^2_{ (U_{h,\hat{\phi}}^{(n)})^{-1} }\right]^{\frac{\alpha_\Phi}{m+\alpha_\Phi}} \tag{from \pref{thm:is-alpha-smooth} on $\left\|\hat{\phi}_h^{(n)}(s_h,a_h)\right\|^2_{ (U_{h,\hat{\phi}}^{(n)})^{-1} }$}\\
      & \leq 9c\hat{\alpha}_n^2 (dL_\Phi)^{\frac{m}{m+\alpha_\Phi}} n^{\frac{m}{m+\alpha_\Phi}}
     \mathop{\EE}_{s_{h}\sim (P^\star, \Pi_n)\atop a_h\sim \Ucal}
     \left[n\left\|\hat{\phi}_h^{(n)}(s_h,a_h)\right\|^2_{ (U_{h,\hat{\phi}}^{(n)})^{-1} }\right]^{\frac{\alpha_\Phi}{m+\alpha_\Phi}} \\
     & \leq 9c\hat{\alpha}_n^2 (dL_\Phi)^{\frac{m}{m+\alpha_\Phi}} n^{\frac{m}{m+\alpha_\Phi}} d^{\frac{\alpha_\Phi}{m+\alpha_\Phi}} \\
     & = 9c\hat{\alpha}_n^2 d L_\Phi^{\frac{m}{m+\alpha_\Phi}} n^{\frac{m}{m+\alpha_\Phi}}
     \tag{same argument as the original CHYL-L3}
\end{align*}
Putting everything together, we obtain
\begin{align*}
    & \mathop{\EE}_{s_{h}\sim P^\star_{h-1} \atop a_{h} \sim {\pi_n} }\left[\hat{b}^{(n)}_{h}(s_{h},a_{h})\bigg|s_{h-1},a_{h-1}\right] \\
    & \leq \left\|\phi_{h-1}^\star(s_{h-1},a_{h-1})\right\|_{(W_{h-1,\phi^\star}^{(n)})^{-1}}
	\sqrt{9c\hat{\alpha}_n^2 d L_\Phi^{\frac{m}{m+\alpha_\Phi}} n^{\frac{m}{m+\alpha_\Phi}}+\lambda_n d }\\
 & \leq \left\|\phi_{h-1}^\star(s_{h-1},a_{h-1})\right\|_{(W_{h-1,\phi^\star}^{(n)})^{-1}}
    \sqrt{225  
 \beta_3 n^{1+\frac{m}{m+\alpha_\Phi}} \zeta_n^{\frac{\alpha}{m+\alpha}}L_\Phi^{\frac{m}{m+\alpha_\Phi}} d\left( L^{\frac{m}{m+\alpha}} + (d L)^{\frac{m}{m+\alpha}} + d^2\right)
   +\beta_3 n \zeta_n d^2 }\\
 & \leq \left\|\phi_{h-1}^\star(s_{h-1},a_{h-1})\right\|_{(W_{h-1,\phi^\star}^{(n)})^{-1}}
	\sqrt{ 450 
 \beta_3 n^{1+\frac{m}{m+\alpha}} \zeta_n^{\frac{\alpha}{m+\alpha}}L^{\frac{m}{m+\alpha}} d\left( L^{\frac{m}{m+\alpha}} + (d L)^{\frac{m}{m+\alpha}} + d^2\right)
   }\\
 & \leq \gamma_n \left\|\phi_{h-1}^\star(s_{h-1},a_{h-1})\right\|_{(W_{h-1,\phi^\star}^{(n)})^{-1}},
\end{align*}
where we set $\gamma_n = \sqrt{450  
 \beta_3 n^{1+\frac{m}{m+\alpha}} \zeta_n^{\frac{\alpha}{m+\alpha}}L^{\frac{m}{m+\alpha}} d\left( L^{\frac{m}{m+\alpha}} + (d L)^{\frac{m}{m+\alpha}} + d^2\right)
   }$. We note that this coefficient has a similar to the original $\gamma_n$, but with the first $K$ replaced by $L^{\frac{m}{m+\alpha}}$, the second $K$ replaced by $L^{\frac{m}{m+\alpha}}$ and different exponents for $\zeta_n$ and $n$. 

For $h=1$, we similarly have:
\begin{align*}
    \mathop{\EE}_{a_1\sim {\pi_n}}\left[f_1^{(n)}(s_1,a_1)\right] & \leq \sqrt{ c L_E^{\frac{m}{m+\alpha_E}} \mathop{\EE}_{a_1\sim \Ucal}\left[f_1^{(n)}(s_1,a_1)^2\right]^{\frac{\alpha_E}{m+\alpha_E}}} \leq \sqrt{c L_E^{\frac{m}{m+\alpha_E}}\zeta_n^{\frac{\alpha_E}{m+\alpha_E}} }\\
    & \leq \sqrt{c L^{\frac{m}{m+\alpha}}\zeta_n^{\frac{\alpha}{m+\alpha}} },\\
    \mathop{\EE}_{a_1\sim {\pi_n}}\left[\hat{b}(s_1,a_1)\right] & \leq 3\hat{\alpha}_n \sqrt{\mathop{\EE}_{a_1\sim\pi_n}
    \left[\left\|\hat{\phi}_1^{(n)}(s_1,a_1)\right\|^2_{ (U_{1,\hat{\phi}}^{(n)})^{-1} }\right]}\\
    & \leq 15\alpha_n \sqrt{cL_\Phi^{\frac{m}{m+\alpha_\Phi}} \mathop{\EE}_{a_1\sim\Ucal}
    \left[\left\|\hat{\phi}_1^{(n)}(s_1,a_1)\right\|^2_{ (U_{1,\hat{\phi}}^{(n)})^{-1} }\right]^{\frac{\alpha_\Phi}{m+\alpha_\Phi}}} \tag{\pref{thm:is-alpha-smooth}}\\
    & \leq 15\alpha_n \sqrt{c
    \frac{L_\Phi^{\frac{m}{m+\alpha_\Phi}}}{n^{\frac{\alpha_\Phi}{m+\alpha_\Phi}}}d^{\frac{\alpha_\Phi}{m+\alpha_\Phi}}}\\
    & \leq 15\alpha_n \sqrt{cd
    \frac{L_\Phi^{\frac{m}{m+\alpha_\Phi}}}{n^{\frac{\alpha_\Phi}{m+\alpha_\Phi}}}} \leq \leq 15\alpha_n \sqrt{cd
    \frac{L^{\frac{m}{m+\alpha}}}{n^{\frac{\alpha}{m+\alpha}}}}
\end{align*}
This concludes the updated CHYL-L3. We propagate these results through Propositions 4-6 of \citet{cheng2023improved}.

\paragraph{CHYL-P4} For all $n\in[N]$, policy $\pi$ and reward $r$, given that the event $\mathcal{E}$ occurs, we obtain the inequality:
\begin{align*}
    \left| V_{P^\star,r}^{\pi} - V_{\hat{P}^{(n)},r}^{\pi}\right| \leq   \hat{V}_{\hat{P}^{(n)},\hat{b}^{(n)}}^{\pi}+\sqrt{c L^{\frac{m}{m+\alpha}}\zeta_n^{\frac{\alpha}{m+\alpha}} }.
\end{align*}
The proof of this proposition only changes when bounding $ \mathop{\EE}_{a_1\sim {\pi_n}}\left[f_1^{(n)}(s_1,a_1)\right]$ which we instead bound using the updated CHYL-L3 for $h=1$. This concludes the proof. 

\paragraph{CHYL-P5} Replacing the IS arguments and the statements of CHYL-L3 and CHYL-P4 with the updated ones, we obtain:
\begin{align*}
    V_{P^\star, {\hat{b}}^{(n)}}^{{\pi}_n} & \leq \sum_{h=1}^{H} \mathop{\EE}_{s_{h}\sim (P^\star, {\pi_n}) \atop a_{h} \sim {\pi_n}} \left[\alpha_n\left\|\phi_{h}^\star(s_{h},a_{h})\right\|_{(U_{h,\phi^\star}^{(n)})^{-1}}
		\right]
		+\sqrt{c L^{\frac{m}{m+\alpha}}\zeta_n^{\frac{\alpha}{m+\alpha}} }, \\
  V_{P^\star, f^{(n)}}^{{\pi}_n} & \leq \sum_{h=1}^{H} \mathop{\EE}_{s_{h}\sim (P^\star, {\pi_n}) \atop a_{h} \sim {\pi_n}} \left[\alpha_n\left\|\phi_{h}^\star(s_{h},a_{h})\right\|_{(U_{h,\phi^\star}^{(n)})^{-1}}
		\right] + 15\alpha_n \sqrt{cd
    \frac{L^{\frac{m}{m+\alpha}}}{n^{\frac{\alpha}{m+\alpha}}}}.
\end{align*}
Taking the sum of these terms and noting that $\zeta=N\zeta_N$, we have:
\begin{align*}
    &\sum_{n\in[N]} V_{P^\star, f^{(n)} + \hat{b}^{(n)}}^{{\pi}_n} + \sqrt{c L^{\frac{m}{m+\alpha}}\zeta_n^{\frac{\alpha}{m+\alpha}} }\\
    & \qquad \leq 15\alpha_N N \sqrt{cd
    \frac{L^{\frac{m}{m+\alpha}}}{N^{\frac{\alpha}{m+\alpha}}}} 
    + 2 N \sqrt{c L^{\frac{m}{m+\alpha}}\zeta_N^{\frac{\alpha}{m+\alpha}} } + H\gamma_N \sqrt{Nd\zeta} + H\alpha_N \sqrt{cN L_\Phi^{\frac{m}{m+\alpha_\Phi}}(d\zeta)^{\frac{\alpha_\Phi}{m+\alpha_\Phi}}}\\
    & \qquad \leq 15\alpha_N \sqrt{cd L^{\frac{m}{m+\alpha}}N^{1+\frac{m}{m+\alpha}} } + \frac{2\alpha_N\sqrt{N}}{\sqrt{\beta_3(L^{\frac{m}{m+\alpha}}+(dL)^{\frac{m}{m+\alpha}}+d^2)}}\\
    & \qquad \quad + H\sqrt{Nd \zeta}\sqrt{ 450 
 \beta_3 N^{1+\frac{m}{m+\alpha}} \zeta_N^{\frac{\alpha}{m+\alpha}}L^{\frac{m}{m+\alpha}} d\left( L^{\frac{m}{m+\alpha}} + (d L)^{\frac{m}{m+\alpha}} + d^2\right)
   }\\
    & \qquad \quad + H
    \sqrt{\beta_3 N \zeta_N^{\frac{\alpha}{m+\alpha}}\left( L^{\frac{m}{m+\alpha}} + (d L)^{\frac{m}{m+\alpha}} + d^2\right)} \sqrt{cN L^{\frac{m}{m+\alpha}}(d\zeta)^{\frac{\alpha}{m+\alpha}}}\\
    & \qquad \leq 17 \zeta H d 
    \sqrt{\beta_3 N^{1+\frac{2m}{m+\alpha}}L^{\frac{m}{m+\alpha}}\left( L^{\frac{m}{m+\alpha}} + (d L)^{\frac{m}{m+\alpha}} + d^2\right)}\\
    & \qquad \quad + \zeta H d \sqrt{ 450 
 \beta_3 N^{1+\frac{2m}{m+\alpha}} L^{\frac{m}{m+\alpha}}\left( L^{\frac{m}{m+\alpha}} + (d L)^{\frac{m}{m+\alpha}} + d^2\right)
   } \\
    & \qquad \quad + \zeta H d \sqrt{\beta_3 N^{1+\frac{m}{m+\alpha}} L^{\frac{m}{m+\alpha}}\left( L^{\frac{m}{m+\alpha}} + (d L)^{\frac{m}{m+\alpha}} + d^2\right)}\\
    & \qquad \leq (17+\sqrt{450}+1) \zeta H d \sqrt{\beta_3 N^{1+\frac{2m}{m+\alpha}} L^{\frac{m}{m+\alpha}}\left( L^{\frac{m}{m+\alpha}} + (d L)^{\frac{m}{m+\alpha}} + d^2\right)}\\
    & \qquad \leq 64 \zeta H d \sqrt{\beta_3 N^{1+\frac{2m}{m+\alpha}} L^{\frac{m}{m+\alpha}}\left( L^{\frac{m}{m+\alpha}} + (d L)^{\frac{m}{m+\alpha}} + d^2\right)}
\end{align*}

Note that in first inequality above we apply \pref{thm:is-alpha-smooth} first, then we follow the steps as in the original proof of CHYL-L3, and after apply Cauchy Schwartz we use Jensen's inequality to move the $\alpha_E / (m+\alpha_E)$ power out the front, before applying their Lemma 10.

\paragraph{CHYL-P6} We can now leverage the result in CHYL-P6 to obtain the following expression for $N_\epsilon$:
\begin{align*}
    \frac{\epsilon N_\epsilon}{2} & = 64 \zeta H d \sqrt{\beta_3 N_\epsilon^{1+\frac{2m}{m+\alpha}} L^{\frac{m}{m+\alpha}}\left( L^{\frac{m}{m+\alpha}} + (d L)^{\frac{m}{m+\alpha}} + d^2\right)}\\
    \Rightarrow N_\epsilon & = \left(\frac{2^{14}\beta_3 H^2 d^2 L^{\frac{m}{m+\alpha}}\left( L^{\frac{m}{m+\alpha}} + (d L)^{\frac{m}{m+\alpha}} + d^2\right) \zeta^2}{\epsilon^2}\right)^{\frac{1}{1-\frac{2m}{m+\alpha}}}
\end{align*}
Thus, letting $\tau = m/(m+\alpha)<1/2$, the sample complexity of the algorithm is given by $HN_\epsilon$ which can be written as:
\begin{align*}
    \tilde{O}\left(H^{1+\textstyle\frac{2}{1-2\tau}}d^{\textstyle\frac{2}{1-2\tau}}\left(\frac{1}{\epsilon}\right)^{\textstyle\frac{2}{1-2\tau}}L^{\textstyle\frac{\tau}{1-2\tau}}(L^\tau + (Ld)^\tau+d^2)^{\textstyle\frac{1}{1-2\tau}}\right).
\end{align*}
We note that as in the case of FLAMBE, when $\alpha_E, \alpha_\Phi\rightarrow\infty$, \emph{i.e.} $\tau\rightarrow 0$, we recover the sample complexity of RAFFLE with a discrete action space and $K=1$. 
\end{proof}

\end{document}